\documentclass[letterpaper]{article}

\usepackage{aaai20}
\usepackage{times}
\usepackage{helvet}
\usepackage{courier}
\usepackage[hyphens]{url}
\usepackage{Definitions}
\usepackage{amssymb,amsthm,amsmath}
\usepackage{parskip}
\usepackage[capitalise]{cleveref}
\usepackage{multirow}
\usepackage{nicefrac}
\usepackage{microtype}
\usepackage{thmtools, thm-restate}
\usepackage{subfig}
\usepackage{wrapfig}
\usepackage{array,graphicx}

\usepackage{algpseudocode}
\usepackage[ruled,linesnumbered,noline]{algorithm2e}
\usepackage{xr}
\usepackage{tikz}
\usepackage{stmaryrd}
\usepackage{booktabs}
\usepackage{xcolor}

\newcommand{\citet}[1]{\citeauthor{#1} \shortcite{#1}}
\newcommand{\citep}{\cite}

\newcommand\numberthis{\addtocounter{equation}{1}\tag{\theequation}}

\newcommand{\fresh}{{\small\textsc{Fresh}}}

\newcommand{\xhdr}[1]{\noindent{{\bf #1.}}}
\newcommand{\err}{\mathrm{err}}
\newcommand{\ui}{\textsc{ui}}
\newcommand{\ur}{\textsc{ur}}
\newcommand{\ETC}{\textsc{ec}}

\newcommand{\prob}{\mathbf{P}}

\DeclareMathOperator*{\maximize}{\text{maximize\ }}

\newcommand{\R}{\mathbb{R}}
\newcommand{\N}{\mathbb{N}}
\DeclareMathOperator*{\E}{\mathbb{E}}

\newcommand{\cL}{\mathcal{L}}

\newcommand{\given}{\, \vert \,}

\newcommand{\srho}{\rho^\star}
\newcommand{\hrho}{\widehat{\rho}}
\newcommand{\bhrho}{\widehat{\boldsymbol{\rho}}}
\newcommand{\bsrho}{\boldsymbol{\rho}^\star}
\newcommand{\brho}{\boldsymbol{\rho}}

\newcommand{\bkappa}{\boldsymbol{\kappa}}
\newcommand{\bhkappa}{\widehat{\boldsymbol{\kappa}}}
\newcommand{\hxi}{\widehat{\xi}}
\newcommand{\txi}{\widetilde{\xi}}
\newcommand{\hp}{\widehat{p}}
\newcommand{\bhxi}{\widehat{\boldsymbol{\xi}}}
\newcommand{\bxi}{\boldsymbol{\xi}}
\newcommand{\bzeta}{\boldsymbol{\zeta}}

\title{Learning to Crawl}
\title{Learning to Crawl}
\author{
Utkarsh Upadhyay\\
Resonal, Berlin, Germany\\
{utkarsh@reason.al}
\\\And
R{\'o}bert Busa-Fekete\\
Google Research, NY, USA\\
{busarobi@google.com}
\\\And
Wojciech Kot{\l}owski\\
Poznan University of Technology, Poland\\
{wkotlowski@cs.put.poznan.pl}
\\\AND
D{\'a}vid P{\'a}l\\
Yahoo! Research, NY, USA\\
{davidko.pal@gmail.com}
\\\And
Bal{\'a}zs Sz{\"o}r{\'e}nyi\\
Yahoo! Research, NY, USA\\
{szorenyi.balazs@gmail.com}
}

\begin{document}

\maketitle

\begin{abstract}
Web crawling is the problem of keeping a cache of webpages \emph{fresh}, \ie, having the most recent copy available when a page is requested. This problem is usually coupled with the natural restriction that the bandwidth available to the web crawler is limited. The corresponding optimization problem was solved optimally by \citet{Azar8099} under the assumption that, for each webpage, both the elapsed time between two changes and the elapsed time between two requests follows a Poisson distribution with \emph{known} parameters. In this paper, we study the same control problem but under the assumption that the change rates are \emph{unknown} a priori, and thus we need to estimate them in an online fashion using only partial observations (\ie, single-bit signals indicating whether the page has changed since the last refresh). As a point of departure, we characterise the conditions under which one can solve the problem with such partial observability. Next, we propose a practical estimator and compute confidence intervals for it  in terms of the elapsed time between the observations. Finally, we show that the explore-and-commit algorithm achieves an $\Ocal(\sqrt{T})$ regret with a carefully chosen exploration horizon. Our simulation study shows that our online policy scales well and achieves close to optimal performance for a wide range of parameters.
\end{abstract}

\section{Introduction}

As information dissemination in the world becomes near real-time, it becomes more and more important for search engines, like Bing and Google, and other knowledge repositories to keep their caches of information and knowledge fresh.
In this paper, we consider the web-crawling problem of designing policies for refreshing webpages in a local cache with the objective of maximizing the number of incoming requests which are served with the latest version of the page.
Webpages are the simplest and most ubiquitous source of information on the internet.
As items to be kept in a cache, they have two key properties: (i) they need to be polled, which uses bandwidth, and (ii) polling them only provides partial information about their change process, \ie, a single bit indicating whether the webpage has changed since it was last refreshed or not.
\citet{cho2003effective} in their seminal work presented a formulation of the problem which was recently studied by~\citet{Azar8099}.
Under the assumption that the changes to the webpages and the requests are Poisson processes with known rates, they describe an efficient algorithm to find the optimal refresh rates for the webpages.

However, the change rates of the webpages are often not known in advance and need to be estimated.
Since the web crawler cannot continuously monitor every page, there is only partial information available on the change process.
\citet{cho2003estimating}, and more recently~\citet{li2017temporal}, have proposed estimators of the rate of change given partial observations.
However, the problem of learning the refresh rates of items while also trying to optimise the objective of keeping the cache as up-to-date for incoming requests as possible seems very challenging.
On the one hand, the optimal policy found by the algorithm by~\citet{Azar8099} does not allocate bandwidth for pages that are changing very frequently, and, on the other hand, rate estimates with low precision, especially for those that are changing frequently, may result in a policy that has non-vanishing regret.
We formulate this web-crawling problem with unknown change rates as an online optimization problem for which we define a natural notion of regret, describe the conditions under which the refresh rates of the webpages can be learned, and show that using a simple explore-and-commit algorithm, one can obtain regret of order $\sqrt{T}$.

Though in this paper we primarily investigate the problem of web-crawling, our notion of regret and the observations we make about the learning algorithms can also be applied to other control problems which model the actions of agents as Poisson processes and the policies as intensities, including alternate objectives and observation regimes for the web-crawling problem itself.
Such an approach is seen in recent works which model or predict social activities~\citep{farajtabar2018point,du2016recurrent}, control online social actions~\citep{zarezade2017steering,karimi2016smart,wang2017variational,upadhyay2018deep}, or even controlling spacing of items for optimal learning~\citep{tabibian2019enhancing}.
All such problems admit online versions where the parameters for the models (\eg, difficulty of items from recall events~\citep{tabibian2019enhancing}, or rates of posting of messages by other broadcasters~\citep{karimi2016smart}) need to be learned while also optimising the policy the agent follows.

In Section~\ref{sec:formulation}, we will formally describe the problem setup and formulate the objective with bandwidth constraints.
Section~\ref{sec:optimal-policy} takes a closer look at the objective function and the optimal policy to describe the properties any learning algorithm should have.
We propose an estimator for learning the parameters of Poisson process with partial observability and provide guarantees on its performance in Section~\ref{sec:simple-estimator}.
Leveraging the bound on the estimator's performance, we propose a simple explore-and-commit algorithm in Section~\ref{sec:etc} and show that it achieves $\Ocal(\sqrt{T})$ regret.
In Section~\ref{sec:experiments}, we test our algorithm using real data to justify our theoretical findings and we conclude with future research directions in Section~\ref{sec:conclusion}.

\section{Problem Formulation}\label{sec:formulation}

In this section, we consider the problem of keeping a cache of $m$ webpages up-to-date by modelling the changes to those webpages, the requests for the pages, and the bandwidth constraints placed on a standard web-crawler.
We assume that the cache is empty when all processes start at time $0$.

We model the changes to each webpage as Poisson processes with constant rates. The parameters of these \emph{change processes} are denoted by $\bxi = [ \xi_1, \dots , \xi_m ]$, where $\xi_i > 0$ denotes the rate of changes made to webpage $i$. We will assume that $\bxi$ are not known to us but we know only an upper bound $\xi_{\max}$ and lower bound $\xi_{\min}$ on the change rates.
The crawler will learn $\bxi$ by refreshing the pages and observing the single-bit feedback described below.
We denote the time webpage $i$ changes for the $n$th time as $x_{i,n}$.
We model the incoming requests for each webpage also as Poisson processes with constant rates and denote these rates as $\bzeta = [\zeta_1, \zeta_2, \dots, \zeta_m ]$. We will assume that these rates, which can also be interpreted as the \emph{importance} of each webpage in our cache, are known to the crawler.
We will denote the time webpage $i$ is requested for the $n$th time as $z_{i,n}$.
The change process and the request process, given their parameters, are assumed to be independent of each other.

We denote time points when page $i$ is refreshed by the crawler using $( y_{i,n} )_{n=1}^\infty$.
The feedback which the crawler gets after refreshing a webpage $i$ at time $y_{i,n}$ consists of a single bit which indicates whether the webpage has changed or not since the last observation, if any, that was made at time $y_{i, n-1}$.
Let $E^{\emptyset}_i[t_0,t]$ indicate the event that neither a change nor a refresh of the page has happened between time $t_0$ and $t$ for webpage $i$.
Define $\fresh(i, t)$ as the event that the webpage is \emph{fresh} in the cache at time $t$.
Defining the maximum of an empty set to be $-\infty$, we have:
\begin{align*}
    &\fresh(i, t) = \nonumber \\
    &\begin{cases}
            0 & \text{ if } E^{\emptyset}_i[0,t] \\
        {1}_{\left(\max \{x_{i,j}: x_{i, j} < t \} < \max \{ y_{i, j}: y_{i, j} < t \} \right)} & \text{ if } \neg E^{\emptyset}_i[0,t]
    \end{cases}
\end{align*}
where the indicator function $1_{(\bullet)}$ takes value 1 on the event in its argument and value 0 on its complement.
Hence, we can describe the feedback we receive upon refreshing a page $i$ at time $y_{i,n}$ as:
\begin{align}
    o_{i,n} = \fresh(i, y_{i, n}).\label{eqn:observation}
\end{align}
We call this a \emph{partial} observation of the change process to contrast it with \emph{full} observability of the process, \ie, when a refresh at $y_{i,n}$ provides the number of changes to the webpage in the period $(y_{i,n}, y_{i,n-1})$. For example, the crawler will have full observability of the incoming request processes.

The policy space $\Pi$ consists of all measurable functions which, at any time $t$, decide when the crawler should refresh which page in its cache based on the observations up to time $t$ that includes $\{ (o_{i,n})_{n=1}^{N} \given N=\argmax_n y_{i,n}<t; i \in [m] \}$. %

The objective of the web-crawling problem is to refresh webpages such that it maximizes the number of requests which are served a fresh version.
So the utility of a policy $\pi \in \Pi$ followed from time $t_1$ to $t_2$ can be written as:
\begin{align}\label{eq:utility}
    U( [t_1, t_2], \pi; \bxi ) = \frac{1}{m} \sum_{i=1}^m \sum_{t_1 \le z_{i,n} \le t_2} \fresh(i, z_{i,n}).
\end{align}

Our goal is to find a policy that maximizes this utility (\ref{eq:utility}).\footnote{The freshness of the webpages does depend on the policy $\pi$ which is hidden by function $\fresh$.}
However, if the class of policies is unconstrained, the utility can be maximized by a trivial policy which continuously refreshes all webpages in the cache. This is not a practical policy since it will overburden the web servers and the crawler. %
Therefore, we would like to impose a bandwidth constraint on the crawler. %
Such a constraint can take various forms and a natural way of framing it is that the expected number of webpages that are refreshed in \emph{any} time interval with width $w$ cannot exceed $w \times R$.
This constraint defines a class of stochastic policies
$\Delta_{R} = \left\{ (\rho_1, \dots , \rho_m)\in (\R^+)^m: \sum_{i=1}^m\rho_i = R \right\} \subset \Pi$, where each webpage's refresh time is drawn by the crawler from a Poi\-sson process with rate $\rho_i$.

This problem setup was studied by~\citet{Azar8099} and shown to be tractable.
Recently,~\citet{kolobov2019staying} have studied a different objective function which penalises the \emph{harmonic staleness} of pages and is similarly tractable.
Another example of such an objective (albeit with full observability) is proposed by~\citet{sia2007efficient}.
We show later that our analysis extends to their objective as well.

We define the regret of policy $\pi \in \Delta_R$ as follows
\begin{align}
R(T, \pi; \bxi) = \max_{\pi' \in \Delta_R} \E \left[U( [0, T], \pi'; \bxi )\right] - \E \left[ U( [0, T], \pi; \bxi ) \right].\nonumber
\end{align}

It is worth reiterating that the parameters $\bxi$ will not be known to the crawler. The crawler will need to determine when and which page to refresh given only the single bits of information $o_{i,n}$ corresponding to each refresh the policy makes.

\section{Learning with Poisson Processes and Partial Observability}\label{sec:optimal-policy}

In this section, we will derive an analytical form of the utility function which is amenable to analysis, describe how to uncover the optimal policy in $\Delta_R$ if all parameters (\ie, $\bxi$ and $\bzeta$) are known, and consider the problem of learning the parameters $\bxi$ with partial observability. We will use the insight gained to determine some properties a learning algorithm should have so that it can be analysed tractably.

\subsection{Utility and the Optimal policy}

Consider the expected value of the utility of a policy $\brho \in \Delta_R$ which the crawler follows from time $t_0$ till time $T$.
Assume that the cache at time $t_0$ is given by $\Sbb(t_0) = [s_1, s_2, \dots, s_m] \in \{ 0, 1 \}^m$, where $s_i = \fresh(i, t_0)$.
Then, using~\eqref{eq:utility}, we have:
\begin{align}
    &\EE [ U([t_0, T], \brho; \bxi) \mid \Sbb(t_0) ]   \nonumber \\
    &= \frac{1}{m} \sum_{i=1}^m \EE\bigg[\sum_{t_0 < z_{i,n} < T}  \fresh(i, z_{i,n})\,
    \bigg| %
    \, \fresh(i, t_0) = s_i \bigg] \nonumber \\ %
    &= \frac{1}{m} \sum_{i=1}^{m} \int_{t_0}^T \underbrace{\zeta_i \prob_{\brho}\left( \fresh(i, t)=1 \mid \fresh(i, t_0) = s_i \right)}_{F^{(i)}_{[t_0,t]}(\brho; \bxi)} \,dt \label{eqn:exp-utility-3}
\end{align}
where~\eqref{eqn:exp-utility-3} follows from Campbell's formula for Poisson Process
\citep{poissonprocess}
(expectation of a sum over the point process equals the integral over time with process'
intensity measure) as well as the fact that the request process and change/refresh processes are independent.\footnote{Note that the \emph{rates} of the processes can be correlated; only the events need to be drawn independently.}
In the next lemma, we show that the differential utility function $F^{(i)}_{[t_0, t]}(\brho; \bxi)$, defined implicitly in~\eqref{eqn:exp-utility-3}, can be made time-independent if the policy is allowed to run for \emph{long-enough}.

\begin{restatable}[Adapted from~\citep{Azar8099}]{lem}{burnin}
\label{lemma:burn-in}
For any given $\varepsilon > 0$, let $\brho \in \Delta_R$ be a policy which the crawler adopts at time $t_0$ and let the initial state of the cache be $\Sbb(t_0) = [s_1, s_2, \dots, s_m] \in \{ 0, 1 \}^m$, where $s_i = \fresh(i, t_0)$. Then if
$t - t_0 \ge \frac{1}{\xi_{\min}}\log\left(\frac{2\sum_{j=1}^{m} \zeta_i}{\varepsilon} \right)$, then $\sum_{i=1}^m \frac{\zeta_i \xi_i}{\xi_i + \rho_i} < \sum_{i=1}^{m} F^{(i)}_{[t_0, t]}(\brho; \bxi)  < \sum_{i=1}^m \frac{\zeta_i \xi_i}{\xi_i + \rho_i} + \varepsilon$.
\end{restatable}

Hence, as long as condition described by Lemma~\ref{lemma:burn-in} holds, the differential utility function for a policy $\brho \in \Delta_R$ is time independent and can be written as just $F(\brho; \bxi) = \sum_{i=1}^m \frac{\zeta_i \xi_i}{\xi_i + \rho_i}$. Substituting this into~\eqref{eqn:exp-utility-3}, we get:
\begin{align}
   \EE[ U([t_0, T], \brho; \bxi) ]  &\approx \frac{1}{m} \sum_{i=1}^{m} \int_{t_0}^T \frac{\rho_i}{\rho_i + \xi_i} \zeta_i\,dt \nonumber \\
   &= \frac{T - t_0}{m} \sum_{i = 1}^{m} \frac{ \rho_i \zeta_i }{\rho_i + \xi_i} \label{eqn:utility-approx}
\end{align}
This leads to the following time-horizon independent optimisation problem for the optimal policy:
\begin{align}
    \maximize_{\bm{\rho}\in \Delta_R} & F(\brho; \bxi) = \sum_{i = 1}^{m} \frac{ \rho_i \zeta_i }{\rho_i + \xi_i}
    \label{eqn:azar-optimal}
\end{align}
\citet{Azar8099} have considered the approximate utility function given by~\eqref{eqn:azar-optimal} to derive the optimal refresh rates $\bsrho$ for known $\bxi$ in $\Ocal(m \log{m})$ time (See Algorithm 2 in~\citep{Azar8099}).\footnote{The optimal policy can be obtained in $\Ocal(m)$ time by using the method proposed by~\citet{duchi2008efficient}.}

This approximation has bearing upon the kind of learning algorithms we could use while keeping the analysis of the algorithm and computation of the optimal policy tractable.
The learning algorithm we employ must follow a policy $\brho \in \Delta_R$ for a certain amount of \emph{burn-in} time before we can use~\eqref{eqn:utility-approx} to approximate the performance of the policy. If the learning algorithm changes the policy \emph{too quickly}, then we may see large deviations between the actual utility and the approximation.
However, if the learning algorithm changes the policy \emph{slowly}, where Lemma~\ref{lemma:burn-in} can serve as a guide to the appropriate period, then we can use~\eqref{eqn:utility-approx} to easily calculate its performance between $t_0$ and $T$.
Similar conditions can be formulated for the approximate objectives proposed by~\citet{kolobov2019staying} and~\citet{sia2007efficient} as well.

Now that we know how to uncover the optimal policy when $\bxi$ are known, we turn our attention to the task of learning it with partial observations.

\subsection{Learnability with Partial Observations}

In this section, we address the problem of partial information of Poisson process and investigate under what condition the rate of the Poisson process can be estimated.
In our setting, for an arbitrary webpage, we only observe binary outcomes $(o_{n})_{n=1}^{\infty}$, defined by~\eqref{eqn:observation}.
The refresh times $( y_{n} )_{n=1}^\infty$ and the Poisson process of changes with rate $\xi$ induce a distribution over $\{0,1\}^{\N}$ which is denoted by $\mu_{\xi}$.
If the observations happen at regular time intervals, i.e. $y_{n}-y_{n-1} = c$ for some constant $c$, then the support $\Scal_{\xi}$ of $\mu_{\xi}$ is:
\begin{align*}
    \Scal_{\xi} = \left\{ (o_{n})_{n=1}^{\infty} \in \{0,1 \}^{\N}: \lim_{n\rightarrow\infty} \frac{\sum_{j=1}^{n} o_{j}}{n} = 1-e^{-c\, \xi} \right\}.
\end{align*}
This means that we can have a consistent estimator, based on the strong law of large numbers, if the crawler refreshes the cache at fixed intervals.

However, we can characterise the necessary property of the set of partial observations which allows parameter estimation of Poisson processes under the general class of policies $\Pi$. This result may be of independent interest.

\begin{restatable}[]{lem}{identifiability}
\label{lemma:identifiablity}
Let $\{ y_0 := 0 \} \cup (y_n)_{n=1}^\infty$ be a sequence of times, such that $\forall n.\, y_n > 0$, at which observations $(o_n)_{n=1}^\infty \in \{0, 1\}^{\NN}$ are made of a Poisson process with rate $\xi$, such that $o_n := 1$ iff there was an event of the process in $(y_{n - 1}, y_{n}]$, define $w_n = y_{n} - y_{n - 1}$, $I=\{n: w_n<1\}$ and $J = \{n: w_n\geq 1\}$. Then:
\begin{enumerate}
    \item
    If $\sum_{n \in I} w_n < \infty$ and $\sum_{n \in J} e^{-\xi w_{n}} < \infty$, then %
    any statistic for estimating $\xi$ has non-vanishing bias.
    \item
    If $\sum_{n \in I} w_n = \infty$, then there exist disjoint subsets $I_1, I_2, \dots$ of $I$ such that $\left(\sum_{n \in I_k}w_{n}\right)_{k=1}^{\infty}$ is monotone
    and
    $\sum_{n \in I_k} w_n \in (1,2)$
    for $k=1,2,\dots$ For any such sequence $\mathcal{I} = (I_k)_{k=1}^\infty$,
    the mapping
    $c_{\mathcal I}(\xi) = \lim_{K \to \infty} \frac{1}{K}\sum_{k=1}^K \exp\left({-\xi \sum_{n \in I_k}w_{n} }\right)$ is strictly monotone and
    \begin{align}
       \left[\frac{1}{K}\sum_{k=1}^K \mathbb{I}\left(\sum_{n \in I_k} o_{n}\geq 1\right) \right] \overset{a.s.}{\longrightarrow}
        1-c_{\mathcal I}(\xi)
        .
        \nonumber
    \end{align}

    \item
    If $\sum_{n \in J} e^{-w_{n} \xi} = \infty$ then, there exists a sequence $\mathcal{J} = \left(J_k\right)_{k=1}^\infty$ of disjoint subsets of $J$ such that $\left(\sum_{n \in J_k} e^{-w_{n} \xi}\right)_{k=1}^\infty$ is monotone and $\sum_{n \in J_k} e^{-w_{n} \xi} \in \left[1/e,{2/e}\right)$ for $k=1,2,\dots$
    For any such $\mathcal J$, the mapping $c_{\mathcal J}(\xi) = \lim_{K \to \infty}\left[ \frac{1}{K} \sum_{k=1}^K\prod_{n \in J_k} \left(1-e^{-\xi w_{n} }\right)\right]$ is strictly monotone and
    \begin{align}
       \lim_{K \to \infty} \left[\frac{1}{K} \sum_{k=1}^K\mathbb{I}\left(o_{n} {\geq} 1,\; \forall n \in J_k\right) \right] \overset{a.s.}{\longrightarrow}
       c_{\mathcal J}(\xi),\nonumber
    \end{align}
\end{enumerate}
\end{restatable}

Note that it is possible that, for some $\xi$, the statistics almost surely converge to a value that is unique to $\xi$, but for some other one they do not. Indeed, when $w_n = \ln n$, then $\sum_{n \in I} w_n < \infty$ and $\sum_{n \in J} e^{-\xi w_{n}} < \infty$ for $\xi = 2$, but $\sum_{n \in J} e^{-\xi w_{n}} = \infty$ for $\xi = 1$.
More concretely, assuming that the respective limits exist, we have:
\begin{align}
  \liminf_{n \in J}{\frac{w_n}{\ln{n}}} > \frac{1}{\xi} &\implies \sum_{n} e^{-\xi w_n} < \infty  \quad \text{ and }\nonumber \\
  \quad
  \limsup_{n \in J}{\frac{w_n}{\ln{n}}} \le \frac{1}{\xi} &\implies  \sum_{n} e^{-\xi w_n} = \infty.
  \nonumber
\end{align}
In particular, if $\limsup_{n \in J} \frac{w_n}{\ln{n}} = 0$, it implies that
$\sum_{n \in J} e^{-\xi w_n} = \infty$ for all $\xi > 0$, which, in turn, implies that it will be possible to learn the true value for any parameter $\xi > 0$.

Lemma~\ref{lemma:identifiablity} has important implications on the learning algorithms we can use to learn $\bxi$.
It suggests that if the learning algorithm decreases the refresh rate $\rho_i$ for a webpage too quickly, such that $\prob\left(\lim\inf_{n \to \infty} \frac{w_{i,n}}{\ln n} > \frac{1}{\xi_i} \right) > 0$ (assuming the limit exists), then the estimate of each parameter $\xi_i$ has non-vanishing error.

In summary, in this section, we have made two important observations about the learning algorithm we can employ to solve the web-crawling problem. Firstly, given an error tolerance of $\varepsilon > 0$, the learning algorithm should change the policy only after $\sim \frac{1}{\xi_{\min}}\log{\left( \nicefrac{2\sum_i \zeta_i}{\varepsilon} \right)}$ steps to allow for time invariant differential utility approximations to be valid.
Secondly, in order to obtain consistent estimates for $\bxi$ from partial observations, the learning algorithm should not change the policy so drastically that it violates the conditions in Lemma~\ref{lemma:identifiablity}.
These observations strongly suggest that to obtain theoretical guarantees on the regret, one should use \emph{phased} learning algorithms where each phase of the algorithm is of duration $\sim \frac{1}{\xi_{\min}}\log{(\nicefrac{2\sum_i \zeta_i}{\varepsilon})}$, the policy is only changed when moving from one phase to the other, and the changes made to the policy are such that constraints presented in Lemma~\ref{lemma:identifiablity} are not violated.
Parallels can be drawn between such learning algorithms and the algorithms used for online learning of Markov Decision Processes which rely on bounds on mixing times~\citep{neu2010online}.
In Section~\ref{sec:etc}, we present the simplest of such algorithms, \ie, the explore-and-commit algorithm, for the problem and provide theoretical guarantees on the regret. Additionally, in Section~\ref{sec:eps}, we also empirically compare the performance of ETC to the phased $\varepsilon$-greedy learning algorithm.

In the next section, we investigate practical estimators for the parameters $\bhxi$ and the formal guarantees they provide for the web-crawling problem.

\section{Parameter Estimation and Sensitivity Analysis with Partial Observations}\label{sec:simple-estimator}

In this section, we address the problem of parameter estimation of Poisson processes under partial observability and investigate the relationship between the utility of the optimal policy $\bsrho$ (obtained using true parameters) and policy $\bhrho$ (obtained using the estimates).

Assume the same setup as for Lemma~\ref{lemma:identifiablity}, \ie, we are given a finite sequence of observation times $\{ y_0 := 0 \}\, \cup\, (y_{n})_{n=1}^{N}$ in advance, and we observe $(o_{n})_{n=1}^{N}$, defined as in~\eqref{eqn:observation}, based on a Poisson process with rate $\xi$. Define $w_n = y_n - y_{n - 1}$. Then log-likelihood of $(o_{n})_{n=1}^{N}$ is:
\begin{align}
\cL (\xi ) %
&= \sum_{n: o_n = 1} \ln ( 1- e^{-\xi w_n } ) - \sum_{n: o_n = 0} \xi w_{n}\label{eqn:ll-est}
\end{align}
which is a concave function. Taking the derivative and solving for $\xi$ yields the maximum likelihood estimator~\citep{cho2003estimating}.
However, as the MLE estimator lacks a closed form, coming up with a non-asymptotic confidence interval is a very challenging task.
Hence, we consider a simpler estimator.

Let us define an intermediate statistic $\hp$ as the fraction of times we observed that the underlying Poisson process produced no events,
$\hp = \frac{1}{N} \sum_{n=1}^{N} (1-o_n)$. Since $\prob(o_n = 0) = e^{-\xi w_n}$ we get
$\EE[\hp] = \frac{1}{N} \sum_{n=1}^{N} e^{-{\xi} w_n}$.
Motivated by this, we can estimate $\xi$ by the following moment matching method: choose $\txi$ to be the unique solution of the equation
\begin{align}
 \hp = \frac{1}{N} \sum_{n=1}^{N} e^{-\txi w_n}
 \label{eqn:xihat-definition},
\end{align}
and then obtain estimator $\hxi$ of $\xi$ by clipping $\txi$ to range $[\xi_{\min}, \xi_{\max}]$,
$\hxi = \max\{\xi_{\min}, \min\{\xi_{\max}, \txi\}\}$.
The RHS in~\eqref{eqn:xihat-definition} is monotonically decreasing in $\txi$, therefore finding the solution of~\eqref{eqn:xihat-definition} with error $\gamma$ can be done in $O ( \log (1/\gamma))$ time based on binary search.
Additionally, if the intervals are of fixed size, \ie, $\forall n.\, w_n = c$, then $\txi$ reduces to the maximum likelihood estimator.
Such an estimator was proposed by~\citet{cho2003estimating} and was shown to have good empirical performance. Here, instead of smoothing the estimator, a subsequent clipping of $\txi$ resolves the issue of its instability for the extreme values of $\hp=0$ and $\hp=1$ (when the solution to \eqref{eqn:xihat-definition} becomes $\txi =\infty$ and $\txi =0$, respectively).
In the following lemma, we will show that this estimator $\hxi$ is also amenable to non-asymptotic analysis by providing a high probability confidence interval for it.

\begin{restatable}[]{lem}{concentration}
\label{lemma:concentration}
Under the condition of Lemma~\ref{lemma:identifiablity}, for any $\delta \in (0, 1)$, and $N$ observations it holds that
\begin{align*}
\prob \left( | \hxi - \xi | \ge \left(\frac{1}{N}\sum_{n=1}^N w_n e^{-\xi_{\max} w_n} \right)^{-1}
\sqrt{\frac{\log \frac{2}{\delta}}{2N}} \right)
\le \delta
\end{align*}
where $\hxi = \max\{\xi_{\min}, \min\{\xi_{\max}, \txi\}\}$ and
$\txi$ is obtained by solving~\eqref{eqn:xihat-definition}.
\end{restatable}

Similar analysis can be done for the setting when one has full observability~\citep[See Appendix K]{upadhyay2019learning}.
With the following lemma we bound the sensitivity of the expected utility to the accuracy of our parameter estimates $\bhxi$. %
\begin{restatable}[]{lem}{sensitivity}
\label{lemma:sensitivity}
For the expected utility $F(\brho; \bxi)$ defined in~\eqref{eqn:azar-optimal}, let $\bsrho = \argmax_{\brho} F(\brho; \bxi)$, $\bhrho = \argmax_{\brho} F(\brho; \bhxi)$ and define the suboptimality of $\bhrho$ as $\err(\bhrho) := F(\bsrho; \bxi) - F(\bhrho; \bxi)$. Then $\err(\bhrho) $ can be bounded by:
\begin{align*}
\err(\bhrho) \leq  \sum_i\frac{1}{\hxi_i \min\{\hxi_i, \xi_i\}} \zeta_i (\hxi_i - \xi_i)^2.
\end{align*}
\end{restatable}

This lemma gives us hope that if we can learn $\bhxi$ well enough such that $| \hxi_i - \xi_i | \sim \Ocal\left( \nicefrac{1}{\sqrt{T}} \right)$ for all $i$, then we can obtain sub-linear regret by following the policy $\bhrho = \argmax_{\brho} F(\brho; \bhxi)$.
This indeed is possible and, in the next section, we show that an explore-and-commit algorithm can yield $\Ocal(\sqrt{T})$ regret.
We would like to also bring to the notice of the reader that a similar result up to constants can be shown for the Harmonic staleness objective proposed by~\citet{kolobov2019staying} (see~\citep[Appendix E]{upadhyay2019learning}) and the accumulating delay objective by~\citet{sia2007efficient} (see~\citep[Appendix F]{upadhyay2019learning}).

\section{Explore-Then-Commit Algorithm}\label{sec:etc}

In this section, we will analyse a version of the explore-and-commit (ETC) algorithm for solving the web-crawling problem.
The algorithm will first learn $\bxi$ by sampling all pages till time $\tau$ and then commit to the policy of observing the pages from time $\tau$ till $T$ at the rates $\bhrho$ as given by the Algorithm 2 in~\citep{Azar8099}, obtained by passing it the estimated rates $\bhxi$ instead of the true rates $\bxi$.

\xhdr{Revisiting the policy space} The constraint we had used to define the policy space $\Delta_R$ was that given any interval of width $w$, the expected number of refreshes in that interval should not exceed $wR$, which limited us to the class of Poisson policies. However, an alternative way to impose the constraint is to bound the time-averaged number of requests made per unit of time asymptotically.
It can be shown that given our modelling assumptions that request and change processes are memory-less, the policy which maximizes the utility in~\eqref{eq:utility} given a fixed number of observations per page will space them equally.
This motivates a policy class $\Kcal_R = \{ \bkappa = (\kappa_1, \dots, \kappa_m) : \sum_{i=1}^m \nicefrac{1}{\kappa_i} = R \} \subset \Pi$ as the set of deterministic policies which refresh webpage $i$ at regular time intervals of length $\kappa_i$.
Policies from $\Kcal_R$ allow us to obtain tight confidence intervals for $\bhxi$ by a straight-forward application of Lemma~\ref{lemma:concentration}.
However, the sensitivity of the utility function for this policy space to the quality of the estimated parameters is difficult to bound tightly.
In particular, the differential utility function for this class of policies (defined in~\eqref{eqn:exp-utility-3}) is not strongly concave, which is a basic building block of Lemma~\ref{lemma:sensitivity}. This precludes performance bounds which are quadratic in the error of estimates $\bhxi$, which lead to worse bounds on the regret of the ETC algorithm.
These reasons are further expounded in the extended version of the paper~\citep[see Appendix I]{upadhyay2019learning}.
Nevertheless, we can show that using the uniform-intervals policy $\bkappa^{\ui}$ incurs \emph{lower} regret than the uniform-rates policy $\brho^{\ur}$, while still making on average $R$ requests per unit time~\citep[see Appendix H]{upadhyay2019learning}.

Hence, to arrive at regret bounds, we will perform the exploration using \emph{Uniform-interval exploration} policy $\bkappa^{\ui} \in \Kcal_R$ which refreshes webpages at regular intervals $\forall i.\, \kappa_i = \frac{m}{R}$, which will allow us to use Lemma~\ref{lemma:concentration} to bound the error of the estimated $\bhxi$ with high probability.
\begin{restatable}[]{lem}{uiparameterconcentration}
\label{lemma:uiparameterconcentration}
For a given $\delta \in (0, 1)$, after following the uniform-interval policy $\bkappa^{\ui}$ for time
$\tau$, which is assumed to be a multiplicity of $\nicefrac{m}{R}$,
we can claim the following for the error in the estimates $\bhxi$ produced using the estimator proposed in Lemma~\ref{lemma:concentration}:
    \begin{align}
    \prob \left( \forall i \in [m] \colon
    | \hxi_i - \xi_i | \le e^{\frac{\xi_{\max} m}{R}} \sqrt{\frac{R\log{\frac{2m}{\delta}}}
    {2\tau m}}
    \right) \ge 1- \delta.
    \nonumber
    \end{align}
\end{restatable}

With these lemmas, we can bound the regret suffered by the ETC algorithm using the following Theorem.
\begin{restatable}[]{thm}{roottregret}
\label{thm:root-t-regret}
Let $\pi^{\ETC}$ denote the explore-and-commit algorithm which explores using the uniform-interval exploration policy for time $\tau$ (assumed to be a multiplicity of $\frac{m}{R}$), estimates $\bhxi$ using the estimator proposed in~\eqref{eqn:xihat-definition}, and then uses the policy $\bhrho = \argmax_{\brho \in \Delta_R} F(\brho; \bhxi)$ till time $T$. Then for a given $\delta \in (0, 1)$, with probability $1 - \delta$, the expected regret of the explore and commit policy $\pi^{\ETC}$ is bounded by:
\begin{align}
     R(T, \pi^{\ETC}; \bxi)
         &\le  \left( \frac{\tau}{m} +
         \frac{(T - \tau)}{\tau} \frac{e^{2\frac{\xi_{\max} m}{R}} R \log{(\nicefrac{2m}{\delta})}}{2m^2 \xi_{\min}^2} \right) \sum_{i = 1}^m \zeta_i.
         \nonumber
\end{align}

Further, we can choose an exploration horizon $\tau^\star \sim \Ocal(\sqrt{T})$ such that, with probability $1 - \delta$, the expected regret is $\Ocal(\sqrt{T})$.
\end{restatable}
\begin{proof}
Since the utility of any policy is non-negative, we can upper-bound the regret of the algorithm in the exploration phase by the expected utility of the best stationary policy $\bsrho = \argmax_{\brho \in \Delta_R} F(\brho; \bxi)$, which is $\frac{\tau}{m} \sum_{i=1}^m  \zeta_i \frac{\srho_i}{\srho_i + \xi_i} < \frac{\tau}{m} \sum_{i=1}^m \zeta_i$. In the exploitation
phase, the regret is given by $\frac{T-\tau}{m} (F(\brho^*,\bxi) - F(\bhrho, \bxi))$ (see \eqref{eqn:utility-approx}), which we bound using Lemma \ref{lemma:sensitivity}.
Hence,
we see that (with a slight abuse of notation to allow us to write $R(T, \bkappa^{\ui}; \bxi)$ for $\bkappa^{\ui} \in \Kcal_R$):
\begin{align}
    R(T, \pi^{\ETC}&; \bxi) = R(\tau, \bkappa^{\ui}; \bxi) + R(T - \tau, \bhrho; \bxi) \nonumber\\
    &\le \frac{\tau}{m} \sum_{i = 1}^m \zeta_i + \frac{(T - \tau)}{m} \sum_{i = 1}^m \frac{\zeta_i (\hxi_i - \xi_i)^2}{\hxi_i \min\{\hxi_i, \xi_i\}} \label{eqn:regret-decomposition-1}
\end{align}

As we are using the estimator from Lemma~\ref{lemma:concentration}, we have $\hxi_i \min\{\hxi_i, \xi_i\} \ge \xi^2_{\min}$. Using this and Lemma~\ref{lemma:uiparameterconcentration} with~\eqref{eqn:regret-decomposition-1}, we get
with probability $1-\delta$:
\begin{align}
 R(T, &\pi^{\ETC}; \bxi)  \nonumber \\
 &\le \tau  \overbrace{\frac{1}{m}\sum_{i = 1}^m \zeta_i}^A + (T - \tau)\frac{\sum_{i = 1}^m \zeta_i }{m \xi_{\min}^2}(\hxi_i - \xi_i)^2 \nonumber\\
 &= A \tau  + (T - \tau) \frac{\sum_{i = 1}^m \zeta_i}{m \xi_{\min}^2} \left( e^{\frac{\xi_{\max} m}{R}} \sqrt{\frac{R \log{(\nicefrac{2m}{\delta})}}{2m\tau}} \right)^2 \nonumber\\
 &= A \tau  + \frac{(T - \tau)}{\tau} \underbrace{\frac{\sum_{i = 1}^m \zeta_i}{2m^2 \xi_{\min}^2} e^{2\frac{\xi_{\max} m}{R}} R \log{(\nicefrac{2m}{\delta})}}_{B} \nonumber \\
 &= A \tau + \frac{B\,T}{\tau} - B \label{eqn:simplified-regret-bound}
\end{align}
This proves the first claim.

The bound in~\eqref{eqn:simplified-regret-bound} takes the minimum value when $\tau^\star = \sqrt{\frac{B}{A}} \sqrt{T} = \sqrt{ \frac{e^{2\frac{\xi_{\max} m}{R}} R \log{(\nicefrac{2m}{\delta})}}{2m \xi_{\min}^2}} \sqrt{T}$, giving with probability $1 - \delta$, the worst-case regret bound of:
\begin{align*}
 R(T, \brho^{\ETC}; \bxi) &< 2 \sqrt{ABT}
\end{align*}

This proves the second part of the theorem.
\end{proof}

This theorem bounds the expected regret conditioned on the event that the crawler learns $\bhxi$ such that $\forall i.\, |\hxi_i - \xi_i| < \sqrt{\frac{\log{\nicefrac{2m}{\delta}}}{\nicefrac{2\tau R}{m}}}$.
These kinds of guarantees have been seen in recent works~\citep{rosenski2016multi,avner2014concurrent}.

Note that the proof of the regret bounds can be easily generalised to the full-observation setting~\citep[Appendix K]{upadhyay2019learning} and for other objective functions~\citep[see Appendix E and F]{upadhyay2019learning}.

Finally, note that using the doubling trick the regret bound can be made horizon independent at no extra cost.
The policy can be de-randomized to either yield a fixed interval policy in $\Kcal_R$ or, to a carousel like policy with similar performance guarantees~\cite[See Algorithm 3]{Azar8099}.
With this upper-bound on the regret of the ETC algorithm, in the next section we explore the empirical performance of the strategy.

\section{Experimental Evaluation}\label{sec:experiments}

We start with the analysis of the ETC algorithm, which shows empirically that the bounds that we have proven in Theorem~\ref{thm:root-t-regret} are tight up to constants.
Next, we compare the ETC algorithm with phased $\varepsilon$-greedy algorithm and show that phased strategies can out-perform a well-tuned ETC algorithm, if given sufficient number of phases to learn. We leave the detailed analysis of this class of algorithms for later work.
An empirical evaluation of the MLE estimator and the moment matching estimator for partial observations, and the associated confidence intervals proposed in Lemma~\ref{lemma:concentration} is done in the extended version of the paper~\citep[Appendix J]{upadhyay2019learning}. These show that, for a variety of different parameters, the performance of the MLE estimator and the moment matching estimator is close to each other.

\subsection{Evaluation of ETC Algorithm}

\begin{figure}[t]
  \centering
  \subfloat[Regret / $\tau$ (with $T=10^4$)]{\includegraphics[width=0.23\textwidth]{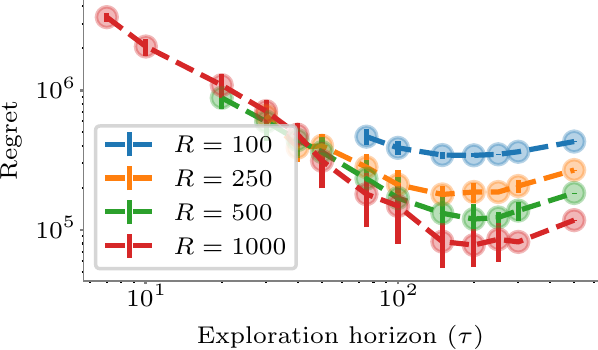}\label{fig:regret-tau}}\hspace{0.01\textwidth}%
  \subfloat[$\tau^\star$ / $T$]{\includegraphics[width=0.23\textwidth]{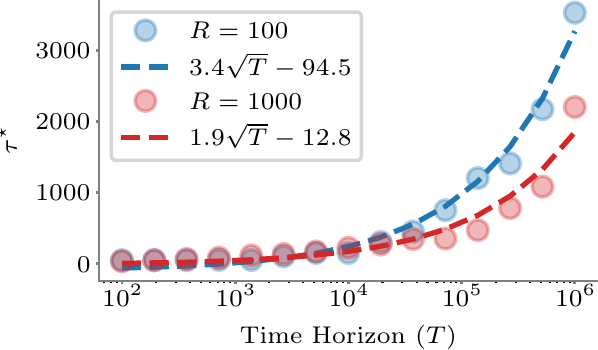}\label{fig:tau-T}}\hspace{0.01\textwidth}%
  \subfloat[Normalized Regret / $T$]{\includegraphics[width=0.23\textwidth]{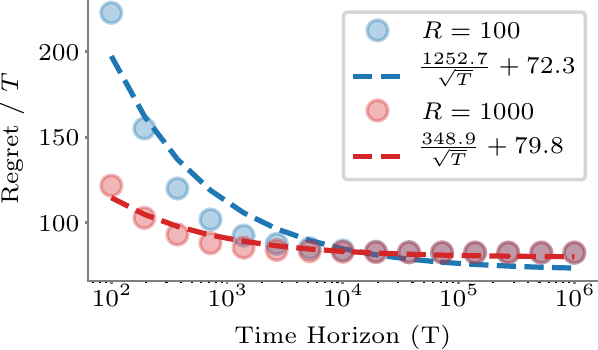}\label{fig:regret-T}}\hspace{0.01\textwidth}%
  \caption{Performance of the ETC algorithm. Panel (a) shows the regret suffered with different exploration horizons (keeping $T=10^4$ fixed) showing that a minima exists. Panel (b) shows that the optimal value of the horizon scales as $\Ocal(\sqrt{T})$ while panel (c) shows that the time-horizon normalized regret of the ETC algorithm decreases as $\Ocal(\nicefrac{1}{\sqrt{T}})$.}
\end{figure}

For the experimental setup, we make user of the MSMACRO dataset~\citep{kolobov2019optimal}.
The dataset was collected over a period of 14 weeks by the production crawler for Bing.
The crawler visited webpages approximately once per day. The crawl time and a binary indicator of whether the webpage had changed since the last crawl or not are included in the dataset along with an importance score for the various URLs.
We sample 5000 webpages from the dataset while taking care to exclude pages which did not change at all or which changed upon every crawl so as to not constraint the estimates of the change rates artificially to $\xi_{\min}$ or $\xi_{\max}$.
We calculate their rate of change ($\bxi$) using the MLE estimator~\eqref{eqn:ll-est}.
The corresponding importance values $\bzeta$ are also sampled from the importance value employed by the production web-crawler.
We set $\xi_{\min} = 10^{-9}$ and $\xi_{\max} = 25$.
The experiments simulate the change times $\left( (x_{i,n})_{n=1}^\infty \right)_{i \in [m]}$ for webpages %
50 times with different random seeds and report quantities with standard deviation error bars, unless otherwise stated.

We first empirically determine the regret for different exploration horizon $\tau$ and bandwidth parameter $R$.
To this end, we run a grid search for different values of $\tau$ (starting from the minimum time required to sample each webpage at least once), simulate the exploration phase using uniform-interval policy $\bkappa^{\ui}$ and simulated change times to determine the parameters $\bhxi$ and the regret suffered during the exploration phase.
We calculate $\bhrho$ using Algorithm 2 of~\citet{Azar8099}, similarly calculate $\bsrho$ using the true parameters $\bxi$, calculate their respective utility after the commit phase from $\tau$ till time horizon $T = 10^4$ using~\eqref{eqn:utility-approx}, and use it to determine the total regret suffered.
We report the mean regret with the error bars indicating the standard deviations in Figure~\ref{fig:regret-tau}.
We see that there indeed is an optimal exploration horizon, as expected, and the value of both the horizon and the regret accumulated depends on $R$. We explore the relationship between the optimal exploration horizon $\tau^\star$ and the time horizon $T$ next by varying $T$ from $10^2$ to $10^6$ and calculating the optimal horizon $\tau^\star$ (using ternary search to minimize the empirical mean of the utility) for $R \in \{10^2, 10^3\}$; plots for other values of $R$ are qualitatively similar.
Figure~\ref{fig:tau-T} shows that the optimal exploration horizon $\tau^\star$ scales as $\Ocal(\sqrt{T})$.

Finally, we plot the time-horizon normalized regret suffered by $\pi^{\ETC}$ when using the optimal exploration horizon $\tau^\star$ in Figure~\ref{fig:regret-T}.
We see that the normalized regret decreases as $\frac{1}{\sqrt{T}}$, as postulated by~\cref{thm:root-t-regret}.
Plots for different values of $\bxi$ and $\bzeta$ are qualitatively similar. %
It can also be seen in all plots that if the allocated bandwidth $R$ is high, then the regret suffered is lower but the dependence of the optimal exploration threshold $\tau^\star$ on the available bandwidth is non-trivial: in Figure~\ref{fig:tau-T}, we see that the $\tau^\star_{R=100} < \tau^\star_{R = 1000}$ if $T < 10^3$ and $\tau^\star_{R=100} > \tau^\star_{R = 1000}$ if $T > 10^4$.

It is noteworthy that Figures~\ref{fig:tau-T} and~\ref{fig:regret-T} suggest that the bounds we have proven in Theorem~\ref{thm:root-t-regret} are tight up to constants.

\subsection{Phased $\varepsilon$-greedy Algorithm}\label{sec:eps}

In Section~\ref{sec:etc}, we have shown guarantees on the performance of explore-then-commit algorithm which is the simplest form of policies which adheres to the properties we mention in Section~\ref{sec:optimal-policy}.
However, it is not difficult to imagine other strategies which conform to the same recommendations.
For example, consider a phased $\varepsilon$-greedy algorithm which runs with $\brho^{\ur}$ for duration given in~\cref{lemma:burn-in}, estimates $\bhxi$, calculates $\bhrho$, and then follows the policy $\brho^\varepsilon$, where $\rho^\varepsilon_i = (1 - \varepsilon) \hrho_i + \varepsilon \frac{R}{m}$, and then starts another phase, improving its policy with improving estimates of $\bhxi$.
Since $\forall i\in[m].\, \rho^\varepsilon_i > \varepsilon\frac{R}{m}$, the policy will continue exploring all the webpages, ensuring eventual consistency in the estimates of $\bhxi$.

We performed simulations with the $\brho^\varepsilon$ algorithm and found that though it performed worse than ETC for small time horizons (Figures~\ref{fig:3-phases} and~\ref{fig:6-phases}), it performed better when given sufficient number of phases (see Figure~\ref{fig:9-phases}).
Exploring the regret bounds of such policies is part of our future work.

\begin{figure}[t]
  \centering
  \subfloat[$T = 10^{3.67}$ / 3 phases]{\includegraphics[width=0.23\textwidth]{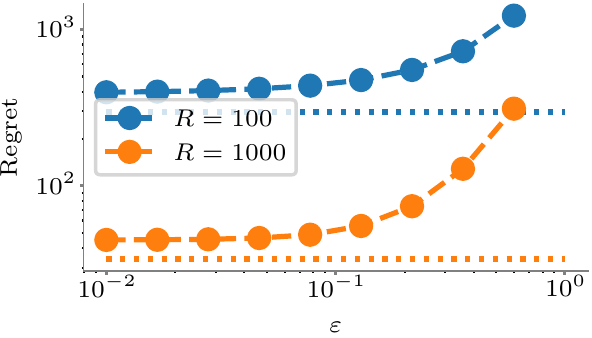}\label{fig:3-phases}}\hspace{0.01\textwidth}%
  \subfloat[$T = 10^{3.83}$ / 6 phases]{\includegraphics[width=0.23\textwidth]{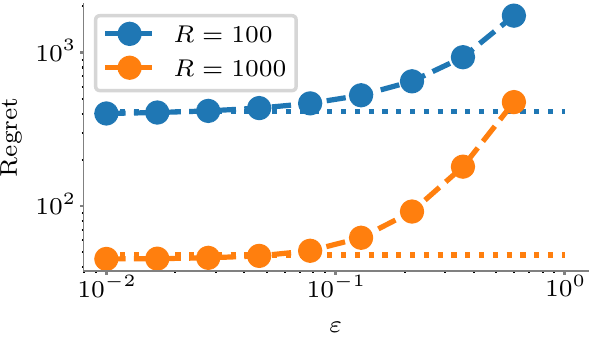}\label{fig:6-phases}}\hspace{0.01\textwidth}%
  \subfloat[$T = 10^4$ / 9 phases]{\includegraphics[width=0.23\textwidth]{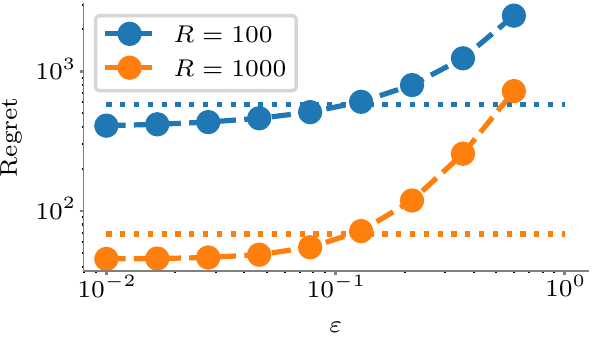}\label{fig:9-phases}}\hspace{0.01\textwidth}%
  \caption{Performance of the phased $\varepsilon$ greedy algorithm. The dotted lines show the regret of the ETC algorithm, with optimal exploration horizon, same bandwidth $R$ and time horizon $T$. While the ETC algorithm performs well when the number of phases is small, with increasing number of phases, the phased $\varepsilon$-greedy algorithm is able to obtain lower regret than ETC.}
\end{figure}

\section{Conclusion}\label{sec:conclusion}

In this paper, we have taken the first step towards solving the problem of learning changing rates of web-pages while solving the web-crawling problem, while also providing a guiding framework for analysing online learning problems where the agent's policy can be described using a Poisson process.
We have shown that the learning algorithms should be \emph{phased} and there are restrictions on how much they can change the policy from step to step while keeping learning feasible.
Further, by bounding the performance of a Poisson policy using a deterministic policy, we have proved that a simple explore-and-commit policy has $\Ocal(\sqrt{T})$ regret under mild assumptions about the parameters.

This leaves several interesting avenues of future work open.
Though we have proved a theoretical upper bound on regret of $\Ocal(\sqrt{T})$ and have empirically seen that bound is tight up to constants for the explore-and-commit algorithm, it is not clear whether this bound is tight for the class of all \emph{phased} strategies. We will explore the class of such strategies in a planned extension.
Lastly, we believe there are rich connections worth exploring between this work and the recent work on the Recharging Bandits paradigm~\citep{immorlica2018recharging}.

\section*{Acknowledgements}
W. Kot{\l}owski was supported by the Polish National Science Centre under grant No. 2016/22/E/ST6/00299.

\bibliographystyle{aaai}
\bibliography{biblio}

\newpage{}

\begin{appendix}

\section{Proof of~\cref{lemma:burn-in}}
\burnin*
\begin{proof}
Let $E_i^\emptyset[t_0,t]$ denote an event that neither a change nor a refresh has happened
for webpage $i$ in time interval $[t_0,t]$. Note that under event $E_i^\emptyset[t_0,t]$, we have $\fresh(i, t) = s_i$. Otherwise (\ie, under event $\neg E_i^\emptyset[t_0,t]$), as we have assumed that the change and the refresh processes are independent Poisson processes for all webpages, the probability that the last event which happened for webpage $i$ between $t_0$ and $t$ was an update event is $\frac{\rho_i}{\rho_i + \xi_i}$.
Hence, we can write the differential utility function as:
\begin{align}
&F^{(i)}_{[t_0, t]}(\brho; \bxi) =     \zeta_i \prob_{\brho}(\fresh(i, t)) \nonumber \\
&= \frac{\zeta_i \rho_i}{\rho_i + \xi_i} \prob(\neg E_i^\emptyset[t_0,t] ) +  \zeta_i s_i \prob(E_i^\emptyset[t_0,t] ) \nonumber \\
    &= \frac{\zeta_i \rho_i}{\rho_i + \xi_i} \left( 1 - e^{-(\rho_i + \xi_i)(t - t_0)} \right) + \zeta_i s_i e^{-(\rho_i + \xi_i)(t - t_0)} \nonumber \\
    &= \frac{\zeta_i \rho_i}{\rho_i + \xi_i} + \left( \frac{\zeta_i \rho_i}{\rho_i + \xi_i} + \zeta_i s_i \right)e^{-(\rho_i + \xi_i)(t - t_0)} \nonumber \\
    &\implies \sum_{i=1}^m F^{(i)}_{[t_0, t]}(\brho; \bxi) \nonumber \\
    &= \sum_{i=1}^m \frac{\zeta_i \rho_i}{\rho_i + \xi_i} + \sum_{i=1}^m \left( \frac{\zeta_i \rho_i}{\rho_i + \xi_i} + \zeta_i s_i  \right)e^{-(\rho_i + \xi_i)(t - t_0)}\label{eqn:fresh-prob-exact-expanded}
\end{align}
This proves the first part of the inequality that $\sum_{i=1}^m \frac{\zeta_i \xi_i}{\xi_i + \rho_i} < \sum_{i=1}^{m} F^{(i)}_{[t_0, t]}(\brho; \bxi)$.

Now substituting $t - t_0 = \frac{1}{\xi_{\min}}\log\left(\frac{2\sum_{j=1}^{m} \zeta_{j}}{\varepsilon} \right)$ into~\eqref{eqn:fresh-prob-exact-expanded}:
\begin{align*}
&\sum_{i=1}^m F^{(i)}_{[t_0, t]}(\brho; \bxi) = \nonumber \\
&\qquad\sum_{i=1}^m \frac{\zeta_i \rho_i}{\rho_i + \xi_i} \\
&\qquad \qquad + \sum_{i=1}^m \left( \frac{\zeta_i \rho_i}{\rho_i + \xi_i}  + \zeta_i s_i\right)e^{-(\rho_i + \xi_i)\frac{1}{\xi_{\min}}\log\left(\frac{2 \sum_{j=1}^{m} \zeta_{j}}{\varepsilon} \right)} \\ %
&\le \sum_{i=1}^m \frac{\zeta_i \rho_i}{\rho_i + \xi_i} + \sum_{i=1}^m \zeta_i \left( \frac{\rho_i}{\rho_i + \xi_i} + s_i \right)\left(\frac{\varepsilon}{2 \sum_{j=1}^{m} \zeta_{j}} \right) \\ %
&\le \sum_{i=1}^m \frac{\zeta_i \rho_i}{\rho_i + \xi_i} + \varepsilon %
\end{align*}
where we have used $\frac{\rho_i + \xi_i}{\xi_{\min}} \ge 1$
in the first inequality and $\frac{\rho_i}{\rho_i + \xi_i} + s_i \le 2$
in the second inequality.
\end{proof}

\section{Proof of~\cref{lemma:concentration}}
\concentration*
\begin{proof}
Recall that $\hp = \frac{1}{N} \sum_{n=1}^N (1-o_n)$ is the empirical frequency of no-event counts, and denote $\EE[\hp]$ by $p$. In this notation, we have:
\[
 \hp = \frac{1}{N} \sum_{n=1}^{N} e^{-\txi w_n},
 \qquad \text{and} \quad
 p = \frac{1}{N} \sum_{n=1}^{N} e^{-\xi w_n},
\]
where $\txi$ is monotonically decreasing in $\hp$ (and similarly for $\xi$ as a function of
$p$).

First, assume that $\hp \le p$, which implies $\txi \ge \xi$ by the monotonicity property mentioned above, and by the property of clipping (and the fact that
$\xi \in [\xi_{\min}, \xi_{\max}]$) we also have $\txi \ge \hxi \ge \xi$.
By the convexity of the exponential function:
\[
e^{-\xi w_n} ~\geq~ e^{-\hxi w_n} - w_n e^{-w_n \hxi} (\xi - \hxi)
~\geq~ e^{-\txi w_n} - w_n e^{-w_n \hxi} (\xi - \hxi),
\]
which implies after summing over $n$:
\begin{equation}
p - \hp ~\ge~ (\hxi - \xi) \frac{1}{N} \sum_{n=1}^{N} w_n e^{-w_n \hxi}
~\ge~ (\hxi - \xi) \frac{1}{N} \sum_{n=1}^{N} w_n e^{-w_n \xi_{\max}}.
\label{eqn:phat-upper-bound}
\end{equation}
Similarly for  $p \leq \hp$ we have $\txi \le \hxi \le \xi$ and therefore:
\[
e^{-\txi w_n} ~\geq~ e^{-\hxi w_n} ~\geq~ e^{-\xi w_n} - w_n e^{-w_n \xi} (\hxi - \xi),
\]
which implies:
\begin{align}
 \hp - p ~\geq~ (\xi - \hxi) \frac{1}{N} \sum_{n=1}^{N} w_n e^{-w_n \xi}
 ~\ge~ (\xi - \hxi) \frac{1}{N} \sum_{n=1}^{N} w_n e^{-w_n \xi_{\max}}
 \label{eqn:phat-lower-bound}
\end{align}
By combining~\eqref{eqn:phat-upper-bound} and~\eqref{eqn:phat-lower-bound},  we get:
\begin{align}
  \left| \hp - p \right|
  \geq |\xi - \hxi| \frac{1}{N} \sum_{n=1}^{n} w_n e^{-w_n \xi_{\max}}
  \label{eqn:phat-xihat-bound}
\end{align}
For $\hp$ being a frequency of counts, Hoeffding's inequality for independent Bernoulli variables implies that for $\delta \in (0, 1)$:
\[
\prob \left(\left| \hp-p \right| < \underbrace{\sqrt{\frac{\log{\nicefrac{2}{\delta}}}{2N}}}_{\varepsilon} \right) \geq 1 - \delta
\]
Hence, with probability at least $1-\delta$, we have $|\hp-p| \leq \varepsilon$ and combining it with~\eqref{eqn:phat-xihat-bound}, with probability $1-\delta$:
\[
|\xi - \hxi| \leq \left(\frac{1}{N} \sum_{n=1}^{N} w_n e^{-w_n \xi_{\max}}\right)^{-1} \varepsilon,
\]
which finishes the proof.
\end{proof}

\section{Proof of~\cref{lemma:identifiablity}}\label{sec:proof-identifiability}
\identifiability*
\begin{proof}
Let $\alpha = -e^{\xi} \log{(1 - e^{-\xi})}$ for which it holds that
\begin{equation}
1-e^{-\xi} = e^{-\alpha e^{-\xi}}.
\label{eq:definition_of_alpha}
\end{equation}
Using the convexity of the exponential function for any $x_0 \in \mathbb{R}$ and
any $\beta \in [0,1]$ we have:
\[
e^{\beta x_0} ~=~ e^{\beta x_0 + (1-\beta) \cdot 0}
~\leq~ \beta e^{x_0} + (1-\beta).
\]
Take any $x \in [0,e^{-\xi}]$ and apply the above to $x_0 = -\alpha e^{-\xi}$
and $\beta = x e^{\xi} \in [0,1]$ to get:
\begin{align}
e^{-\alpha x} &~\leq~ x e^{\xi} e^{-\alpha e^{-\xi}} + (1-xe^{\xi}) \nonumber \\
&~=~ x e^{\xi} (1 - e^{-\xi}) + (1-xe^{\xi})
~=~ 1 - x,
\label{eq: linear upper bound on exp}
\end{align}
where we used~\eqref{eq:definition_of_alpha} in the first equality.

Consider the probability $p$ that $o_{n} = 0$ for all $i \in I$ and $o_{n} = 1$ for all $i \in J$:
\begin{align*}
  p
:=&
\prod_{n \in \NN}
  \prob\left[
    o_n = \begin{cases}1 & \text{ if } n \in J\\
    0& \text{ if } n \in I
    \end{cases}
    \;
  \right] \nonumber \\
&=
  \left(\prod_{n\in I} e^{-w_n \xi}\right)
  \times
  \left(\prod_{n\in J} \left(1-e^{-w_n \xi}\right)\right)
\\
>&\,
   \exp{\left( {- \xi \sum_{n \in I} w_n} \right)}
   \times
   \exp\left(-\alpha\sum_{n \in J} e^{-w_n \xi}\right)
\enspace,
\end{align*}
where the last inequality follows by~\eqref{eq: linear upper bound on exp}.
This probability is bounded away from 0 if $\sum_{n \in I} w_i < \infty$ and $\sum_{n \in J} e^{-w_n \xi} < \infty$.

Now consider another Poisson process with rate $\xi' < \xi$ which we also observe at times $(y_n)_{n=1}^\infty$ to produce observations $(o'_n)_{n=1}^\infty$.
Then define, analogous to $p$ above:
\begin{align*}
 p' := \left(\prod_{n\in I} e^{-w_n \xi'}\right) \times \left(\prod_{n\in J} \left[1-e^{-w_n \xi'}\right]\right),
\end{align*}
which is also positive (since $\xi' < \xi$).
Thus, the two series $(o_{n})_{n=1}^\infty$ and $(o'_{n})_{n=1}^\infty$ are identical with probability at least $p'p > 0$.
This proves the first claim as there is no way to distinguish between $\xi$ and $\xi'$ when $(o_{n})_{n=1}^\infty = (o_{n}')_{n=1}^\infty$.

Regarding the second claim, the existence of the sequence $I_1, I_2, \dots$ satisfying the constraints follows from the monotone subsequence theorem
(\ie, constructing $I_1', I_2', \dots$ that satisfies all the constraints except for the one on monotonicity, the theorem guarantees the existence of a subset that also satisfies the constraint on monotonicity).
Note now that, for all $k$:
\begin{align}
\EE \left[\mathbb{I}\left(\sum_{n \in I_k} o_{n}\geq 1\right) \right]
&= 1 - \prob\left(o_n = 0, \forall n \in I_k\right) \nonumber \\
&=
1 - e^{-\xi \sum_{n \in I_k} w_n},
\end{align}
thus, due to the law of large numbers,
\[
\frac{1}{K}\sum_{k=1}^K\left[\mathbb{I}\left(\sum_{n \in I_k} o_{n}\geq 1\right) - \left( 1 - e^{-\xi \sum_{n \in I_k} w_n} \right)\right]
\overset{a.s.}{\longrightarrow} 0
.
\]
Also notice  that $\lim_{K \to \infty}\frac{1}{K}\sum_{k=1}^K\left[ \left( 1 - e^{-\xi \sum_{n \in I_k} w_n} \right)\right]$ exists due to the monotonicity constraint and the fact that $\xi$ and each $w_n$ are positive.
These together imply
\begin{align}
\frac{1}{K}\sum_{k=1}^K&\left[\mathbb{I}\left(\sum_{n \in I_k} o_{n}\geq 1\right) \right]
\overset{a.s.}{\longrightarrow} \nonumber \\
&\lim_{K \to \infty}\frac{1}{K}\sum_{k=1}^K\left[  1 - e^{-\xi \sum_{n \in I_k} w_n}\right] \nonumber \\
&=
1 - \lim_{K \to \infty}\frac{\sum_{k=1}^K e^{-\xi \sum_{n \in I_k} w_n} }{K}
.
\end{align}
Due to the constraint $\sum_{n \in I_k} w_n < 2$, \begin{equation}
\label{eq: limes bounded away from 0}
\lim_{K \to \infty}\frac{\sum_{k=1}^K e^{-\xi \sum_{n \in I_k} w_n} }{K} > 0 .
\end{equation}
Finally, the constraint $\sum_{n \in I_k} w_n > 1$ implies that, for any $ \xi > \xi' > 0$,
\[
\xi \sum_{n \in I_k} w_n - \xi' \sum_{n \in I_k} w_n
=
(\xi-\xi') \sum_{n \in I_k} w_n
>
\xi-\xi',
\]
further implying
\begin{equation}
\label{eq: monotonicity of limes}
\frac{\sum_{k=1}^K e^{-\xi' \sum_{n \in I_k} w_n} }{K}
>
e^{\xi-\xi'}\frac{\sum_{k=1}^K e^{-\xi \sum_{n \in I_k} w_n} }{K}
.
\end{equation}
Strict monotonicity, and thus also uniqueness, now follows from \eqref{eq: monotonicity of limes} and \eqref{eq: limes bounded away from 0}.

The last claim follows similarly.
\end{proof}

\section{Proof of~\cref{lemma:sensitivity}}\label{sec:proof-sensitivity}
\sensitivity*
\begin{proof}
Since $\bsrho$ minimizes $F(\brho, \bxi)$,
it follows from the first-order optimality condition that:
\begin{align*}
    \left(\nabla_{\bsrho} F(\bsrho, \bxi)\right)^\top (\bhrho - \bsrho) &\leq 0\\
     \Longrightarrow
    \sum_i \frac{\zeta_i \xi_i}{(\xi_i + \srho_i)^2} (\hrho_i - \srho_i) &\leq 0.
\end{align*}
We thus lower-bound the suboptimality by:
\begin{align}
    \err(\bhrho) &~=~ F(\bsrho; \bxi) - F(\bhrho; \bxi) \nonumber \\
    &\geq F(\bsrho; \bxi) - F(\bhrho; \bxi) + \left(\nabla_{\bsrho} F(\bsrho, \bxi)\right)^\top (\bhrho - \bsrho)  \nonumber \\
    &= \sum_i \left( \frac{\zeta_i \srho_i}{\xi_i + \srho_i}
    - \frac{\zeta_i \hrho_i}{\xi_i + \hrho_i} +
    \frac{\zeta_i \xi_i}{(\xi_i + \srho_i)^2} (\hrho_i - \srho_i) \right) \nonumber \\
    &= \sum_i \zeta_i \left( \frac{\xi_i(\srho_i- \hrho_i)}{(\xi_i + \srho_i)(\xi_i + \hrho_i)}
    - \frac{\xi_i (\srho_i - \hrho_i)}{(\xi_i + \srho_i)^2} \right) \nonumber \\
    &= \sum_i \frac{\zeta_i \xi_i(\srho_i - \hrho_i)^2}{(\xi_i + \srho_i)^2(\xi_i + \hrho_i)}.
    \label{eq:lower_bound_on_err}
\end{align}
On the other hand, because $\bhrho$ minimizes $F(\brho, \bhxi)$, we
can upper-bound the suboptimality by:
\begin{align}
 \err(\bhrho) &= F(\bsrho; \bxi) - F(\bsrho; \bhxi) + \nonumber
  \underbrace{F(\bsrho; \bhxi) - F(\bhrho; \bhxi)}_{\leq 0} \nonumber \\
  &\quad + F(\bhrho; \bhxi) - F(\bhrho; \bxi) \nonumber \\
  &\leq F(\bsrho; \bxi) - F(\bsrho; \bhxi) ~+~
  F(\bhrho; \bhxi) - F(\bhrho; \bxi).
    \label{eq:upper_bound_on_err}
\end{align}
We explicitly calculate both differences on the right hand side:
\begin{align*}
F(\bsrho; \bxi) - F(\bsrho; \bhxi)
&= \sum_i \zeta_i \left( \frac{\srho_i}{\xi_i + \srho_i} - \frac{\srho_i}{\hxi_i + \srho_i} \right) \nonumber \\
&= \sum_i  \frac{\zeta_i \srho_i(\hxi_i - \xi_i)}{(\xi_i + \srho_i)(\hxi_i + \srho_i)}, \\
F(\bhrho; \bhxi) - F(\bhrho; \bxi)
&= \sum_i \zeta_i \left( \frac{\hrho_i}{\hxi_i + \hrho_i} - \frac{\hrho_i}{\xi_i + \hrho_i} \right) \nonumber \\
&= \sum_i  \frac{\zeta_i \hrho_i(\xi_i - \hxi_i)}{(\xi_i + \hrho_i)(\hxi_i + \hrho_i)},
\end{align*}
so that \eqref{eq:upper_bound_on_err} becomes
\begin{align*}
 &\err(\bhrho) \\
 &\leq \sum_i \zeta_i (\hxi_i - \xi_i) \left(
  \frac{\srho_i}{(\xi_i + \srho_i)(\hxi_i + \srho_i)}
  - \frac{\hrho_i}{(\xi_i + \hrho_i)(\hxi_i + \hrho_i)}
  \right) \\
    &= \sum_i
  \frac{\zeta_i (\xi_i \hxi_i - \srho_i \hrho_i) (\hxi_i - \xi_i)(\srho_i - \hrho_i) }{(\xi_i + \srho_i)(\hxi_i + \srho_i)(\xi_i + \hrho_i)(\hxi_i + \hrho_i)} \\
   &\leq \sum_i
  \frac{\zeta_i |\xi_i \hxi_i - \srho_i \hrho_i| |\hxi_i - \xi_i||\srho_i - \hrho_i| }{(\xi_i + \srho_i)(\hxi_i + \srho_i)(\xi_i + \hrho_i)(\hxi_i + \hrho_i)}.
\end{align*}
Further bound:
\[
    |\xi_i \hxi_i - \srho_i \hrho_i|
    \leq \xi_i \hxi_i + \srho_i \hrho_i \leq (\hxi_i + \srho_i)(\xi_i + \hrho_i),
\]
to obtain:
\begin{align*}
 \err(\bhrho) &\leq \sum_i
  \frac{\zeta_i |\hxi_i - \xi_i||\srho_i - \hrho_i|}{(\xi_i + \srho_i)(\hxi_i + \hrho_i)} \\
  &= \sum_i \sqrt{
  \frac{\zeta_i (\xi_i + \hrho_i)(\hxi_i - \xi_i)^2}{\xi_i(\hxi_i + \hrho_i)^2}
  }
  \sqrt{
  \frac{\zeta_i \xi_i(\srho_i - \hrho_i)^2}{(\xi_i + \srho_i)^2(\xi_i + \hrho_i)}
  } \\
  &\leq \sqrt{ \sum_i
  \frac{\zeta_i (\xi_i + \hrho_i)(\hxi_i - \xi_i)^2}{\xi_i(\hxi_i + \hrho_i)^2}
  }
  \sqrt{ \sum_i
  \frac{\zeta_i \xi_i(\srho_i - \hrho_i)^2}{(\xi_i + \srho_i)^2(\xi_i + \hrho_i)}
  } \\
  &\leq \sqrt{ \sum_i
  \frac{\zeta_i (\xi_i + \hrho_i)(\hxi_i - \xi_i)^2}{\xi_i(\hxi_i + \hrho_i)^2}
  } \sqrt{\err(\bhrho)},
\end{align*}
where in the last but one inequality we used Cauchy-Schwarz inequality $\sum_i x_i y_i
\leq \sqrt{\sum_i x_i^2} \sqrt{\sum_i y_i^2}$, while in the last inequality
we used \eqref{eq:lower_bound_on_err}. Solving for $\err(\bhrho)$ gives:
\begin{align*}
\err(\bhrho) &\leq
\sum_i  \frac{\zeta_i (\xi_i + \hrho_i)(\hxi_i - \xi_i)^2}{\xi_i(\hxi_i + \hrho_i)^2} \nonumber \\
&\leq  \sum_i\frac{1}{\hxi_i \min\{\hxi_i, \xi_i\}} \zeta_i (\hxi_i - \xi_i)^2.
\end{align*}
The last inequality follows from: if $\xi_i \leq \hxi_i$,
then $\xi_i + \hrho_i \leq \hxi_i + \hrho_i$ and:
\[
    \frac{\xi_i + \hrho_i}{\xi_i (\hxi_i + \hrho_i)^2}
    \leq  \frac{\hxi_i + \hrho_i}{\xi_i (\hxi_i + \hrho_i)^2} \leq \frac{1}{\xi_i (\hxi_i + \hrho_i)}
    \leq \frac{1}{\xi \hxi_i},
\]
whereas if $\xi_i \geq \hxi_i$ then
\[
    \frac{\xi_i + \hrho_i}{\xi_i (\hxi_i + \hrho_i)^2}
    \leq  \frac{\xi_i + \hrho_i}{\xi_i \hxi_i (\hxi_i + \hrho_i)} \leq \frac{\xi_i}{\xi_i \hxi_i^2}
    = \frac{1}{\hxi_i^2},
\]
and the last inequality follows from the fact that
the function $f(x) = \frac{a + x}{b + x}$ with $a \geq b$
is maximized at $x=0$.
\end{proof}

\section{Sensitivity of Harmonic Staleness to Accuracy of Parameter Estimates}\label{sec:harmonic-staleness}
\newcommand{\bH}{H}

In this section, we consider the Harmonic staleness objective suggested in~\citet{kolobov2019staying}.
We will show that a result similar to Lemma~\ref{lemma:sensitivity} can be arrived at for that objective as well using similar techniques.

To that end, for this section define (from Problem 1 in~\citep{kolobov2019staying}):
\begin{align}
H(\brho, \bxi) = \sum_{i=1}^m \zeta_i \log \frac{\rho_i}{\rho_i + \xi_i}.
\label{eqn:harmonic-staleness}
\end{align}

\begin{restatable}[]{lem}{sensitivity-harmonic}
\label{lemma:sensitivity-harmonic}
For the expected utility $H(\brho; \bxi)$ defined in~\eqref{eqn:harmonic-staleness}, let $\bsrho = \argmax_{\brho} H(\brho; \bxi)$, $\bhrho = \argmax_{\brho} H(\brho; \bhxi)$ and define the suboptimality of $\bhrho$ as $\err(\bhrho) := H(\bsrho; \bxi) - H(\bhrho; \bxi)$. Then $\err(\bhrho) $ can be bounded by:
\begin{align*}
\err(\bhrho) \leq \frac{(\xi_{\min} + R) R^2}{\xi_{\min}^5}
\sum_{i=1}^m \zeta_i (\xi_i - \widehat{\xi}_i)^2.
\end{align*}

\end{restatable}

\begin{proof}

It is easy to show by directly computing second derivatives that $H(\brho, \bxi)$ is concave in $\brho$ and convex in $\bxi$.

As before, let
$\bsrho = \argmax_{\brho \in \Delta_R} H(\brho, \bxi)$ and
$\bhrho = \argmax_{\brho \in \Delta_R} H(\brho, \bhxi)$.
By the optimality conditions,
$\nabla_{\brho} H(\bsrho, \bxi)^\top (\brho - \bsrho) \leq 0$ for all
$\brho \in \Delta_R$.
We also bound the second derivative:
\begin{align*}
\frac{\partial^2 H(\brho, \bxi)}{\partial \rho_i \partial \rho_j}
&= \delta_{ij} \zeta_i \left(\frac{1}{(\rho_i + \xi_i)^2} - \frac{1}{\rho_i^2}\right) \\
&\leq - \delta_{ij}\zeta_i  \frac{\rho_i \xi_i + \xi_i^2}{(\xi_i + \rho_i)^2 \rho_i^2} \\
&= -\delta_{ij} \zeta_i \frac{\xi_i}{(\xi_i + \rho_i) \rho_i^2}  \\
&\leq -\delta_{ij} \zeta_i \underbrace{\frac{\xi_{\min}}{(\xi_{\min} + R) R^2}}_{C}
\end{align*}
where in the last inequality we used the fact that $\frac{\xi_i}{\xi_i + R}$ is increasing
in $\xi_i$.
Taylor expanding $H(\brho, \bxi)$ around $\bsrho$
gives for some
$\bar{\brho}$ on the interval between $\brho$ and $\bsrho$:
\begin{align*}
H(\brho, \bxi) &= H(\bsrho,\bxi) + \underbrace{\nabla_{\brho} H(\bsrho, \bxi)^\top (\brho - \bsrho)}_{\leq 0} \nonumber \\
&\qquad + \frac{1}{2} (\brho - \bsrho)^\top \nabla^2_{\brho} H(\bar{\brho}, \bxi) (\brho - \bsrho) \\
&\leq H(\bsrho,\bxi) - \frac{C}{2}  \sum_{i=1}^m \zeta_i(\rho_i - \srho_i)^2.
\end{align*}
Taking $\brho = \bhrho$, and rearranging, results in:
\[
\err(\bhrho) = H(\bsrho, \bxi) - H(\bhrho, \bxi)
\geq \frac{C}{2}\sum_{i=1}^m \zeta_i(\rho_i - \srho_i)^2.
\]
For the lower bound, note that:
\begin{align}
\err(\bhrho) &
= H(\bsrho, \bxi) - H(\bhrho, \bxi) \nonumber \\
&= H(\bsrho, \bxi) - H(\bsrho, \bhxi) +
\underbrace{H(\bsrho, \bhxi) - H(\bhrho, \bhxi)}_{\leq 0} \nonumber \\
& \qquad + H(\bhrho, \bhxi) - H(\bhrho, \bxi)  \nonumber \\
&\leq \sum_{i=1}^m \zeta_i \log \frac{(\srho_i + \widehat{\xi}_i)
(\widehat{\rho}_i + \xi_i)}{(\srho_i + \xi_i)(\widehat{\rho}_i +
\widehat{\xi}_i)}\nonumber \\
&=  \sum_{i=1}^m \zeta_i \log \left(1 + \frac{
(\srho_i - \widehat{\rho}_i)(\xi_i - \widehat{\xi}_i)}{(\srho_i + \xi_i)(\widehat{\rho}_i + \widehat{\xi}_i)}\right) \nonumber \\
&\leq \sum_{i=1}^m \zeta_i \frac{
(\srho_i - \widehat{\rho}_i)(\xi_i - \widehat{\xi}_i)}{(\srho_i + \xi_i)(\widehat{\rho}_i + \widehat{\xi}_i)} \label{eqn:log1px-ineq-harmonic}\\
&\leq \sum_{i=1}^m \zeta_i \frac{
|\srho_i - \widehat{\rho}_i||\xi_i - \widehat{\xi}_i|}{(\srho_i + \xi_i)(\widehat{\rho}_i + \widehat{\xi}_i)}\nonumber \\
&\leq \frac{1}{\xi_{\min}^2} \sum_{i=1}^m \zeta_i
|\srho_i - \widehat{\rho}_i||\xi_i - \widehat{\xi}_i| \nonumber \\
&\leq \frac{1}{\xi_{\min}^2}
\sqrt{\sum_{i=1}^m \zeta_i (\srho_i - \widehat{\rho}_i)^2}
\sqrt{\sum_{i=1}^m \zeta_i (\xi_i - \widehat{\xi}_i)^2} \label{eqn:cauchy-schwarz-harmonic}\\
&\leq \frac{1}{\xi_{\min}^2} \sqrt{\frac{2}{C}}
\sqrt{\err(\bhrho)}\sqrt{\sum_{i=1}^m \zeta_i (\xi_i - \widehat{\xi}_i)^2}, \nonumber
\end{align}
where in~\eqref{eqn:log1px-ineq-harmonic} we used
$\log(1+x) \leq x$, while in~\eqref{eqn:cauchy-schwarz-harmonic}
we used the Cauchy-Schwarz inequality
$\sum_i x_i y_i \leq \sqrt{\sum_i x_i^2}\sqrt{\sum_i y_i^2}$ with
$x_i = \sqrt{\zeta_i} |\srho_i - \widehat{\rho}_i|$ and
$y_i = \sqrt{\zeta_i} |\xi_i - \widehat{\xi}_i|$.

Solving for $\err(\bhrho)$ gives:
\begin{align*}
\err(\bhrho)
&\leq \frac{2}{C \xi_{\min}^4}
\sum_{i=1}^m \zeta_i (\xi_i - \widehat{\xi}_i)^2 \nonumber \\
&= \frac{(\xi_{\min} + R) R^2}{\xi_{\min}^5}
\sum_{i=1}^m \zeta_i (\xi_i - \widehat{\xi}_i)^2.
\end{align*}

Hence, proved.
\end{proof}

\section{Sensitivity of Accumulated Delay objective to Accuracy of Parameter Estimates}
\label{sec:accumulated-delay}

In this section, we consider the Accumulated Delay objective suggested by~\citet{sia2007efficient}.
The same objective can be arrived at if one formulates the online learning version of the Smart Broadcasting problem as formulated and solved by~\citet{redqueen17wsdm}.

The objective's differential expected utility for the accumulated delay can be expressed as:
\begin{align}
    J(\brho; \bxi) = - \sum_{i=1}^m \frac{\zeta_i \xi_i}{\rho_i}
    \label{eqn:accumulated-delay}
\end{align}

\citet{sia2007efficient} also show that the optimal solution which maximizes the utility~\eqref{eqn:accumulated-delay} under the bandwidth constraint is given by
\begin{align}
    \srho_i &= \frac{R \sqrt{\zeta_i \xi_i}}{\sum_{j=1}^m \sqrt{\zeta_j \xi_j}} \label{eqn:acc-opt-soln}\\
    \implies J(\bsrho; \bxi) &= - \frac{1}{R} \left( \sum_{i=1}^m \sqrt{ \zeta_i \xi_i } \right)^2. \label{eqn:acc-opt-utility}
\end{align}

\begin{restatable}[]{lem}{sensitivity-delay}
\label{lemma:sensitivity-delay}
For the expected utility $J(\brho; \bxi)$ defined in~\eqref{eqn:accumulated-delay}, let $\bsrho = \argmax_{\brho} J(\brho; \bxi)$, $\bhrho = \argmax_{\brho} J(\brho; \bhxi)$ and define the suboptimality of $\bhrho$ as $\err(\bhrho) := J(\bsrho; \bxi) - J(\bhrho; \bxi)$. Then $\err(\bhrho) $ can be bounded by:
\begin{align*}
\err(\bhrho) \leq \frac{\xi_{\max}^2 \left( \sum_{i=1}^m \sqrt{\zeta_i} \right)^4 }{R \zeta_{\min} \xi_{\min}^3} \sum_{i=1}^m (\xi_i - \widehat{\xi}_i)^2.
\end{align*}
\end{restatable}

\begin{proof}
We will follow largely the same steps as in the proof of~\cref{lemma:sensitivity-harmonic} to prove this result as well.
Let $\bsrho = \argmax_{\brho \in \Delta_R} J(\brho; \bxi)$ and $\bhrho = \argmax_{\brho \in \Delta_R} J(\brho; \bhxi)$. By the optimality conditions, $\nabla_{\brho} J(\bsrho)^{T} (\brho - \bsrho) \le 0$ for all $\brho \in \Delta_R$. We can bound the second derivative as well:
\begin{align*}
\frac{\partial^2 J(\brho; \bxi)}{\partial \rho_i \partial \rho_j} &= -2 \delta_{ij} \frac{\zeta_i \xi_i}{\rho_i^3} \\
&\leq -2 \delta_{ij} \frac{\zeta_{\min} \xi_{\min}}{R^3}.
\end{align*}

Taylor expanding $J(\brho; \bxi)$ around $\bsrho$ gives for some point $\bar{\brho}$ between $\brho$ and $\bsrho$
\begin{align*}
J(\brho; \bxi) &= J(\bsrho; \bxi) + \underbrace{\nabla_{\bsrho} J(\brho; \bxi)^{T} (\brho - \bsrho)}_{\le 0} \\
& \qquad + \frac{1}{2} (\brho - \bsrho)^{T} \nabla^2_{\brho} J(\bar{\brho}; \bxi) (\brho - \bsrho) \\
& \leq J(\bsrho; \bxi) - \frac{\zeta_{\min} \xi_{\min} }{R^3} \sum_{i=1}^m (\rho_i - \srho_i)^2
\end{align*}

Take $\brho = \bhrho$, using the closed form optimal solutions for $\bsrho$ and $\bhrho$ from~\eqref{eqn:acc-opt-soln}, using $\bxi$ and $\bhxi$ respectively, and rearranging terms, we have:
\begin{align*}
    \err(\bhrho) &= J(\bsrho; \bxi) - J(\bhrho; \bxi) \\
    &\ge \frac{\zeta_{\min} \xi_{\min}}{R^3} \sum_{i=1}^m \left( \hrho_i - \srho_i \right)^2 \\
    &= \frac{\zeta_{\min} \xi_{\min}}{R^3} \sum_{i=1}^m \left( \frac{R \sqrt{\zeta_i \xi_i}}{\sum_{j=1}^m \sqrt{\zeta_i \xi_j}} - \frac{R \sqrt{\zeta_i \hxi_i}}{\sum_{j=1}^m \sqrt{\zeta_j \hxi_j}}\right)^2 \\
    &= \frac{\zeta_{\min} \xi_{\min}}{R} \sum_{i=1}^m \left( \frac{\sqrt{\zeta_i \xi_i}}{\sum_{j=1}^m \sqrt{\zeta_i \xi_j}} - \frac{\sqrt{\zeta_i \hxi_i}}{\sum_{j=1}^m \sqrt{\zeta_j \hxi_j}}\right)^2 \numberthis \label{eqn:acc-lower-bound}
\end{align*}

Next, for the upper bound on the error, consider
\begin{align*}
    \err(\bhrho) &= J(\bsrho; \bxi) - J(\bsrho; \bhxi) + \underbrace{J(\bsrho; \bhxi) - J(\bhrho; \bhxi)}_{\le 0} \\
    & \qquad + J(\bhrho; \bhxi) - J(\bhrho; \bxi) \\
    &\le J(\bsrho; \bxi) - J(\bsrho; \bhxi) + J(\bhrho; \bhxi) - J(\bhrho; \bxi)\\
    &= - \frac{1}{R} \left(\sum_{i=0}^m \sqrt{\zeta_i \xi_i} \right)^2 - \frac{1}{R} \left(\sum_{i=0}^m \sqrt{\zeta_i \hxi_i} \right)^2 \\
    & \qquad + \frac{1}{R} \left( \sum_{i=1}^m \frac{\zeta_i \hxi_i}{\sqrt{\zeta_i \xi_i}} \right) \left( {\sum_{j=1}^m \sqrt{\zeta_j \xi_j}} \right) \\
    & \qquad + \frac{1}{R} \left( \sum_{i=1}^m \frac{\zeta_i \xi_i}{\sqrt{\zeta_i \hxi_i}} \right) \left( {\sum_{j=1}^m \sqrt{\zeta_j \hxi_j}} \right) \\
    &= \frac{\left( \sum_{j=0}^m \sqrt{\zeta_j \xi_j} \right) \left( \sum_{j=0}^m \sqrt{\zeta_j \hxi_j} \right) }{R} \\
    &\qquad \times \left[ \sum_{i=0}^m \left( \frac{\sqrt{\zeta_i \xi_i}}{\sum_{j=1}^m \sqrt{\zeta_i \xi_j}} - \frac{\sqrt{\zeta_i \hxi_i}}{\sum_{j=1}^m \sqrt{\zeta_j \hxi_j}} \right) \right. \\
    &\qquad \qquad \qquad \left. \times \left( \sqrt{\frac{\hxi_i}{\xi}} - \sqrt{\frac{\xi_i}{\hxi_i}} \right)  \right]\\
\end{align*}

Denote $S = \sum_{j=0}^m \sqrt{\zeta_j \xi_j}$, $\widehat{S} = \sum_{j=0}^m \sqrt{\zeta_j \hxi_j}$ and use Cauchy-Schwarz inequality $\sum_i x_i y_i \leq \sqrt{\sum_i x_i^2}\sqrt{\sum_i y_i^2}$ with
$x_i = \left\vert \frac{\sqrt{\zeta_i \xi_i}}{\sum_{j=1}^m \sqrt{\zeta_i \xi_j}} - \frac{\sqrt{\zeta_i \hxi_i}}{\sum_{j=1}^m \sqrt{\zeta_j \hxi_j}} \right\vert$ and
$y_i = \left\vert \sqrt{\frac{\hxi_i}{\xi}} - \sqrt{\frac{\xi_i}{\hxi_i}} \right\vert$ to get:
\begin{align*}
    \err(\bhrho)
    &\le \frac{S \widehat{S}}{R} \left[ \sum_{i=0}^m \left( \frac{\sqrt{\zeta_i \xi_i}}{S} - \frac{\sqrt{\zeta_i \hxi_i}}{\widehat{S}} \right) \left( \sqrt{\frac{\hxi_i}{\xi}} - \sqrt{\frac{\xi_i}{\hxi_i}} \right)  \right]\\
    &\le \frac{S \widehat{S}}{R}
    \left[ \sum_{i=0}^m \left\vert \frac{\sqrt{\zeta_i \xi_i}}{S} - \frac{\sqrt{\zeta_i \hxi_i}}{\widehat{S}} \right\vert\left\vert \sqrt{\frac{\hxi_i}{\xi}} - \sqrt{\frac{\xi_i}{\hxi_i}} \right\vert  \right]\\
    &\le \frac{S \widehat{S}}{R} \sqrt{ \sum_{i=0}^m \left( \frac{\sqrt{\zeta_i \xi_i}}{S} - \frac{\sqrt{\zeta_i \hxi_i}}{\widehat{S}} \right)^2 } \sqrt{ \sum_{i=0}^m \left( \sqrt{\frac{\hxi_i}{\xi}} - \sqrt{\frac{\xi_i}{\hxi_i}} \right)^2 } \\
    &\le \frac{S \widehat{S}}{R} \sqrt{ \frac{R\,\err(\bhrho)}{\zeta_{\min} \xi_{\min}} } \sqrt{ \sum_{i=0}^m \left( \sqrt{\frac{\hxi_i}{\xi}} - \sqrt{\frac{\xi_i}{\hxi_i}} \right)^2 } \\
\end{align*}
where we have used~\eqref{eqn:acc-lower-bound} in the last inequality.
Because we know $\forall i \in [m].\  \xi_{\min} \le \xi_i, \hxi_i \le \xi_{\max}$, we have $S, \widehat{S} \le \sqrt{\xi_{\max}} \left( \sum_{i=0}^m \sqrt{\zeta_i} \right)$.
\begin{align*}
    \sqrt{\err(\bhrho)} &\le \frac{S \widehat{S}}{\sqrt{R \zeta_{\min} \xi_{\min} }} \sqrt{\sum_{i=0}^m \frac{(\hxi_i - \xi_i)^2}{\xi_i \hxi_i} } \\
    \err(\bhrho) &\le \frac{S^2 \widehat{S}^2}{R \zeta_{\min} \xi_{\min} } \sum_{i=0}^m \frac{(\hxi_i - \xi_i)^2}{\xi_i \hxi_i} \\
    &\le \frac{\xi_{\max}^2 \left( \sum_{i=1}^m \sqrt{\zeta_i} \right)^4 }{R \zeta_{\min} \xi_{\min}^3} \sum_{i=1}^m (\xi_i - \widehat{\xi}_i)^2.
\end{align*}

Hence, proved.
\end{proof}

\section{Proof of~\cref{lemma:uiparameterconcentration}}
\uiparameterconcentration*
\begin{proof}
Running the uniform-interval policy for time $\tau$ results in $N = \frac{\tau R}{m}$ observations collected for each webpage with time intervals
$w_n = \frac{m}{R}$ for all $n = 1,\ldots, N$, including an observation made at $y_{i,0} := 0$, so that
$\frac{1}{N} \sum_{n=1}^N w_n e^{-\xi_{\max} w_n} = \frac{m}{R} e^{-\xi_{\max} m / R}$.
Substituting these in Lemma~\ref{lemma:concentration}, we have that for the $i$-th webpage,
with probability at most $\nicefrac{\delta}{m}$ it holds
\[
|\hxi_i - \xi_i| > \left(\frac{m}{R} e^{-\frac{\xi_{\max} m}{R}}\right)^{-1}
    \sqrt{\frac{\log \frac{2m}{\delta}}{2 \frac{\tau R}{m}}}
    = e^{\frac{\xi_{\max} m}{R}} \sqrt{\frac{R \log\frac{2m}{\delta}}{2 \tau m}}.
\]
By the union bound, the above event occur for any $i \in \{1,\ldots,m\}$ with probability
at most $\delta$, which finishes the proof.
\end{proof}

\section{Utility of $\bkappa^{\ui}$ and $\brho^{\ur}$}\label{sec:ui-regret}

In this section, we show that the uniform-interval exploration policy $\bkappa^{\ui}$ has lower regret than the uniform-rate $\brho^{\ur}$ exploration policy under the assumption that the exploration horizon $\tau$ is a multiple of $\frac{R}{m}$.
The uniform-intervals policy $\bkappa^{\ui}$ regularly refreshes each webpage at regular intervals of size $\frac{m}{R}$. Because $\tau$ is assumed to be a multiplicity of $\frac{m}{R}$, there will be exactly $\frac{\tau}{\nicefrac{m}{R}}$ refreshes made of the webpages.
By using the definition of utility given in~\eqref{eq:utility}, we can show that the expected utility of the uniform-interval policy $\bkappa^{\ui}$ during the exploration phase is given by:
\begin{align}
    \EE[ U([0, \tau], &\bkappa^{\ui}; \bxi) ] \nonumber \\
    &= \frac{1}{m} \sum_{i=1}^m \zeta_i \int_{0}^{\tau} \prob_{\bkappa^{\ui}}(\fresh(i, t))\, dt  \nonumber \\
    &= \frac{1}{m} \sum_{i=1}^m \zeta_i \sum_{j=1}^{ \frac{\tau}{\nicefrac{m}{R}}} \int_{\frac{m(j-1)}{R}}^{\frac{mj}{R}} \prob_{\bkappa^{\ui}}\left( \fresh(i, t) \right)\, dt \nonumber \\
    &= \frac{1}{m} \sum_{i=1}^m \zeta_i \sum_{j=1}^{ \frac{\tau}{\nicefrac{m}{R}}} \int_{0}^{\frac{m}{R}} \prob_{\bkappa^{\ui}}\left( \fresh(i, t) \right)\, dt \nonumber \\
    &= \frac{1}{m} \frac{\tau}{\frac{m}{R}} \sum_{i=1}^m \zeta_i \int_{0}^{\frac{m}{R}} \prob_{\bkappa^{\ui}}\left( \fresh(i, t) \right)\, dt\label{eqn:utility-ui-prob} \\
    &= \frac{1}{m} \frac{\tau}{\frac{m}{R}} \sum_{i=1}^m \zeta_i \int_{0}^{\frac{m}{R}} e^{-\xi_i t} \, dt \label{eqn:utility-ui-explicit}\\
    &= \frac{\tau}{m} \sum_{i=1}^m \zeta_i \frac{(1 - e^{-\xi_i \frac{m}{R}})}{\xi_i \frac{m}{R}} \nonumber %
\end{align}
where the equality between~\eqref{eqn:utility-ui-prob} and~\eqref{eqn:utility-ui-explicit} can be established by seeing that probability a page is fresh at time $t \in [0, \frac{m}{R}]$ is equal to the probability that no change event has occurred in $[0, t]$, \ie, $E^\emptyset_i[0, t]$, which is equal to $\prob(E^\emptyset_i[0, t]) = e^{-\xi_i t}$.

Next, the utility of the uniform-rates policy $\brho^{\ur}$ can be easily calculated under the conditions of the Lemma~\ref{lemma:burn-in} by setting $\forall i.\, \rho_i = \frac{R}{m}$ in~\eqref{eqn:utility-approx} as:
\begin{align*}
    \EE[U([0, \tau], \brho^{\ur}; \bxi)] = \frac{\tau}{m} \times F(\brho^{\ur}; \bxi) = \frac{\tau}{m} \sum_{i=1}^m \frac{\zeta_i}{1 + \xi_i\frac{m}{R}}
\end{align*}
Hence, the regret suffered by the uniform-rate policy during exploration is greater than the regret suffered by the uniform-interval policy:
\begin{align*}
    R(\tau, \brho^{\ur}; \bxi) &- R(\tau, \bkappa^{\ui}; \bxi) \\
    &= \frac{\tau}{m} \sum_{i=1}^m \zeta_i \left( \frac{1 - e^{-\xi_i\frac{m}{R}}}{\xi_i\frac{m}{R}} - \frac{1}{1 + \xi_i\frac{m}{R}} \right) \\
    &\ge 0
\end{align*}
where the inequality follows from $e^x \ge 1 + x$, for any $x \in \RR$.

\section{Exploration using $\brho \in \Delta_R$ and Committing to use $\bkappa \in \Kcal_R$}\label{sec:justification}

To arrive at theoretical guarantees for the ETC algorithm, we perform exploration using uniform intervals policy $\bkappa^{\ui} \in \Kcal_R$ while for the exploitation phase, we use a policy $\bhrho \in \Delta_R$.
In this section, we justify why we refrain from performing exploration using the policy $\brho^{\ur}$ or use a policy %
$\bkappa \in \Kcal_R$ for commit phase.

\subsection{Parameter Estimation using $\brho^{\ur} \in \Delta_R$}

The randomness added during the estimation phase makes it technically very difficult to bound the error in the estimates.
Notice that Lemma~\ref{lemma:concentration} relies on knowing the distribution of the window lengths $w_n$ and the total number of refreshes made $N$.
If these quantities are random, \ie, $N \sim \textrm{Poisson}(\frac{\tau R}{m})$ and $w_n$ are inter-refresh times, then bounding the probability of arriving at accurate estimates $\bhxi$, \ie, $\prob\left( | \hxi_i - \xi_i | < \varepsilon \right)$, becomes intractable.
On the other hand, we can circumvent these problems by using $\bkappa^{\ui} \in \Kcal_R$ (see proof of Lemma~\ref{lemma:uiparameterconcentration}).

\subsection{Sensitivity of Utility to Parameter Estimation Accuracy using $\bhkappa \in \Kcal_R$}

To allow for easy exposition, we will refer to policies in $\bkappa = (\kappa_1, \dots, \kappa_m) \in \Kcal_R$ via the corresponding policy in $\brho = (\rho_1, \dots, \rho_m) \in \Delta_R$, such that $\rho_i = \nicefrac{1}{\kappa_i}$.
Then the differential utility function for policies in $\Kcal_R$ can be written as (see Section~\ref{sec:ui-regret}):
\begin{align*}
G(\brho; \bxi) &= \sum_{i=1}^m \zeta_i \frac{1 - e^{-\nicefrac{\xi_i}{\rho_i}}}{\nicefrac{\xi_i}{\rho_i}}
\end{align*}
This function is also concave with respect to $\brho$ and can be optimized using an algorithm similar to the one proposed by~\citet{Azar8099} or~\citet{duchi2008efficient}, albeit with more involved calculations.

However, we run into a problem while trying to determine how sensitive the utility functions are to errors in estimation of the parameters $\bhxi$.
The second derivative of the utility function is given by
\begin{align}
 \frac{\partial^2 G(\brho; \bxi)}{\partial \rho_i^2} &= - \frac{\zeta_i \xi_i e^{-\nicefrac{\xi_i}{\rho_i}}}{\rho_i^3}
 \label{eqn:G-prime-prime}
\end{align}

Contrasting it with the derivatives for the utility function for policies in $\Delta_R$
\begin{align*}
\frac{\partial^2 F(\brho; \bxi)}{\partial \rho_i^2} &= - \frac{2\zeta_i \xi_i}{(\xi_i + \rho_i)^3}
\end{align*}
reveals that while the second derivative of $F(\brho,\bxi)$ can be bounded if we have a bound $\xi_{\min}$ on the values of $\xi_i$, no such bound can be proposed for the second derivative of $G(\brho; \bxi)$ in~\eqref{eqn:G-prime-prime}, which can be arbitrarily close to zero.
Hence, as the curvature for the objective function $G(\brho; \bxi)$ cannot be bounded, we will not be able to provide a quadratic bound akin to Lemma~\ref{lemma:sensitivity}, in this setting.

\section{Performance of Moment Matching Estimator}\label{sec:moment-estimator}

\begin{figure}[t]
  \centering
  \subfloat[$\xi = 0.15$, $\rho = 0.25$]{\includegraphics[width=0.23\textwidth]{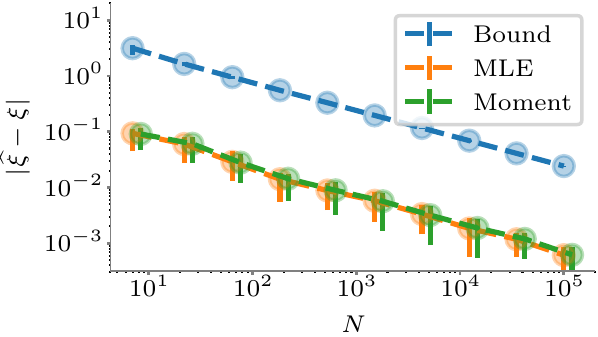}\label{fig:xi-1}}\hspace{0.01\textwidth}%
  \subfloat[$\xi = 0.50$, $\rho = 0.25$]{\includegraphics[width=0.23\textwidth]{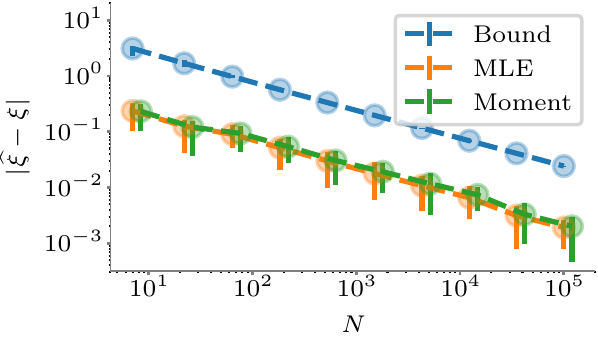}\label{fig:xi-2}}\hspace{0.01\textwidth}%
  \subfloat[$\xi = 0.95$, $\rho = 0.25$]{\includegraphics[width=0.23\textwidth]{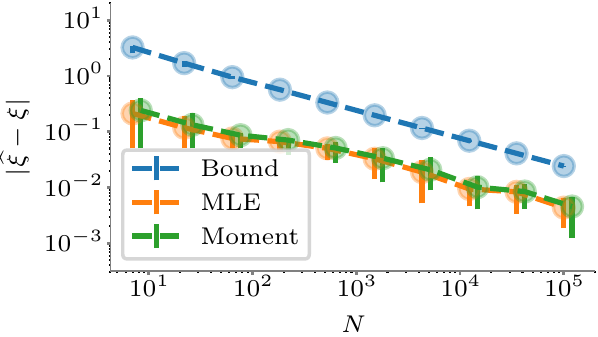}\label{fig:xi-3}}\hspace{0.01\textwidth}%
  \subfloat[$\xi = 0.15$, $\rho = 0.75$]{\includegraphics[width=0.23\textwidth]{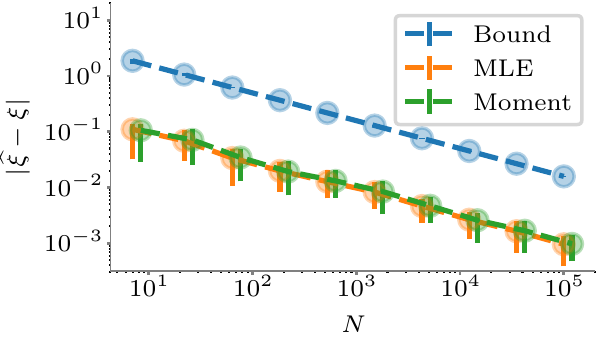}\label{fig:xi-4}}\hspace{0.01\textwidth}%
  \subfloat[$\xi = 0.50$, $\rho = 0.75$]{\includegraphics[width=0.23\textwidth]{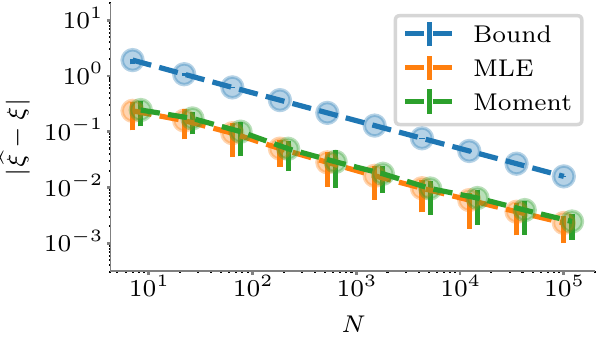}\label{fig:xi-5}}\hspace{0.01\textwidth}%
  \subfloat[$\xi = 0.95$, $\rho = 0.75$]{\includegraphics[width=0.23\textwidth]{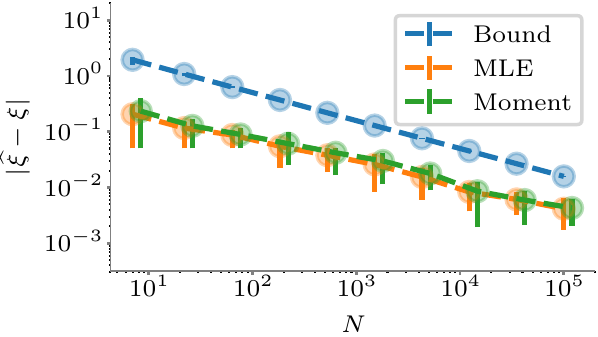}\label{fig:xi-6}}\hspace{0.01\textwidth}%
  \caption{Error in the estimates produced by the moment matching estimator (in green) compared to the upper bound (in blue) and the MLE estimator (in orange) for three different values of $\xi$.
  To calculate the bound, $\xi_{\max}$ was assumed to be $1$.
  The first row shows the estimation error when the refresh rate was given by $\rho = 0.25$ (refresh) events per unit time, while the second row shows the results for $\rho=0.75$ events per unit time.
  The error bars show 25-75 percentiles of error across simulations.}
  \label{fig:moment-estimator}
\end{figure}

In this section, we consider the performance of the moment matching estimator~\eqref{eqn:xihat-definition} and the bounds on its performance proposed in Lemma~\ref{lemma:concentration} on simulated and real data.

In the simulated experiments below, we have assumed that $\xi_{\min} = 0.1$ and $\xi_{\max} = 1.0$.
For a fixed number of observations $N$, known $\bxi$ and a fixed random-seed, we simulate times to refresh a webpage $(y_0 := 0) \cup (y_n)_{n=1}^N$ as times drawn from a Poisson process with rate $\rho \in \{ 0.25, 0.75 \}$. Then we draw $(o_n)_{n=1}^N$ by stochastically determining whether an event happened between $y_{n} - y_{n-1}$.
Next, we calculate the estimates $\hxi$ using the MLE estimator and the moments matching estimator. We also determine what is the bound on the error proposed by Lemma~\ref{lemma:concentration} with $\delta = 0.1$. We record the error in the estimates and then re-run the simulation 50 times with different seeds.
The results are shown in Figure~\ref{fig:moment-estimator}.
Plots for other values of the parameters were similar qualitatively.
It can be seen that the performance of the MLE estimator is not discernibly different from the performance of the moment matching estimator.

\begin{figure}[t]
  \centering
  \subfloat[$\xi = 0.15$, $\rho = 0.25$]{\includegraphics[width=0.23\textwidth]{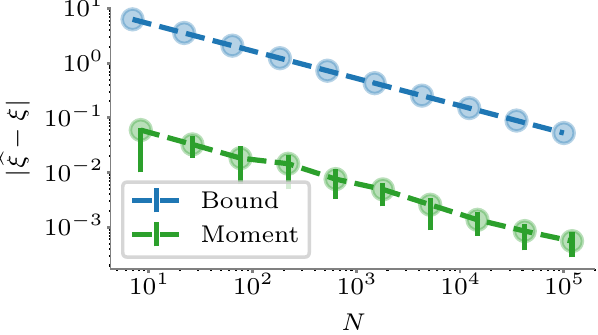}\label{fig:xi-1-det}}\hspace{0.01\textwidth}%
  \subfloat[$\xi = 0.50$, $\rho = 0.25$]{\includegraphics[width=0.23\textwidth]{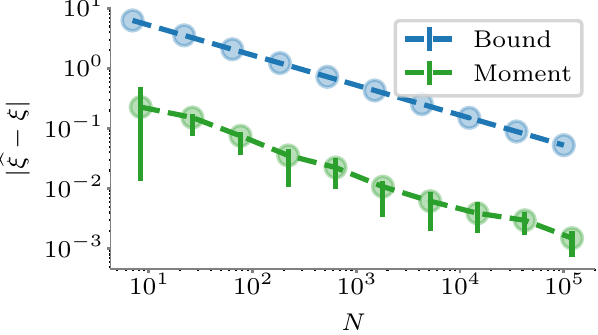}\label{fig:xi-2-det}}\hspace{0.01\textwidth}%
  \subfloat[$\xi = 0.95$, $\rho = 0.25$]{\includegraphics[width=0.23\textwidth]{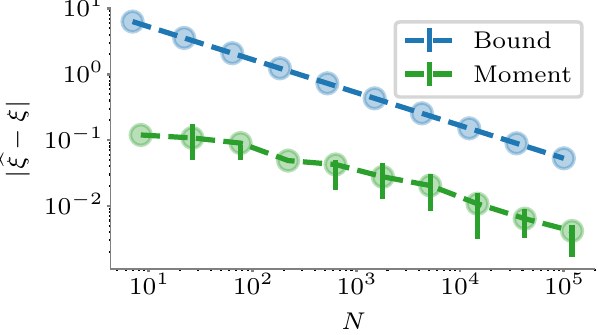}\label{fig:xi-3-det}}\hspace{0.01\textwidth}%
  \subfloat[$\xi = 0.15$, $\rho = 0.75$]{\includegraphics[width=0.23\textwidth]{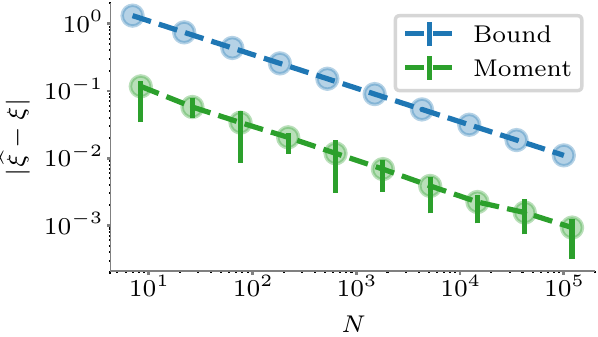}\label{fig:xi-4-det}}\hspace{0.01\textwidth}%
  \subfloat[$\xi = 0.50$, $\rho = 0.75$]{\includegraphics[width=0.23\textwidth]{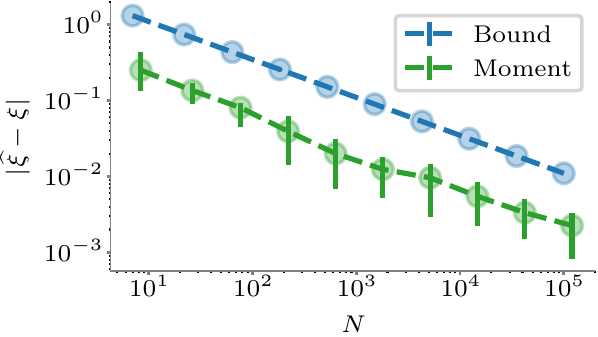}\label{fig:xi-5-det}}\hspace{0.01\textwidth}%
  \subfloat[$\xi = 0.95$, $\rho = 0.75$]{\includegraphics[width=0.23\textwidth]{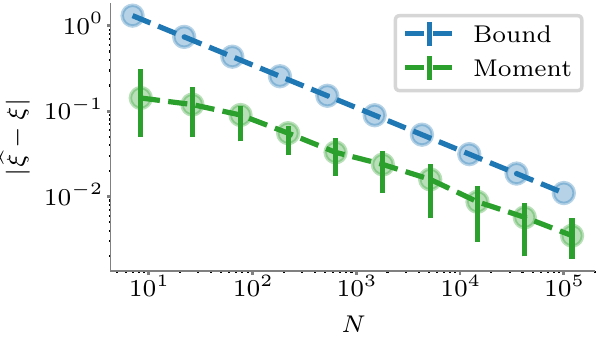}\label{fig:xi-6-det}}\hspace{0.01\textwidth}%
  \caption{Error in the estimates produced by the moment matching estimator for three different values of $\xi$ under the uniform-interval setting. The MLE estimator is not plotted here because it coincides with the moments matching estimator.
  To calculate the bound, $\xi_{\max}$ was assumed to be $1$ and that all intervals are equal with $w_n = \nicefrac{1}{\rho}$.
  The first row shows the estimation error when the refresh rate was given by $\rho = 0.25$ (refresh) events per unit time, while the second row shows the results for $\rho=0.75$ events per unit time.
  The error bars show 25-75 percentiles of error across simulations.
  Compared to Figure~\ref{fig:moment-estimator}, we can see that the bound is slightly looser in the deterministic setting, but the same trend is seen that the bound gets tighter the closer $\xi$ gets to $\xi_{\max}$.
  }
  \label{fig:moment-estimator-det}
\end{figure}

Similarly, we performed the same experiments with fixed-interval policies as well, \ie, when we observed (refreshed) a page $i$ at regular intervals of time $\nicefrac{1}{\rho}$ instead of drawing the times from a Poisson process. Under this setting, both our estimator and the MLE estimator are identical and Figure~\ref{fig:moment-estimator-det} shows the performance under the same setup as before.
In both settings, we can see that the bound decreases with (i) increasing $N$ within each plot, and with (ii) increasing $\rho$, which effects the size of interval $w_n = y_n - y_{n-1}$, across the rows. Also, the bound gets tighter as $\xi$ gets closer to $\xi_{\max}$. Additionally,
These figures also show that the bound is tight up-to constants irrespective of whether the observation (refresh) intervals are stochastic or deterministic, and it is tighter for stochastic intervals than for deterministic intervals.

\begin{figure}[t]
    \centering
    \includegraphics[width=0.35\textwidth]{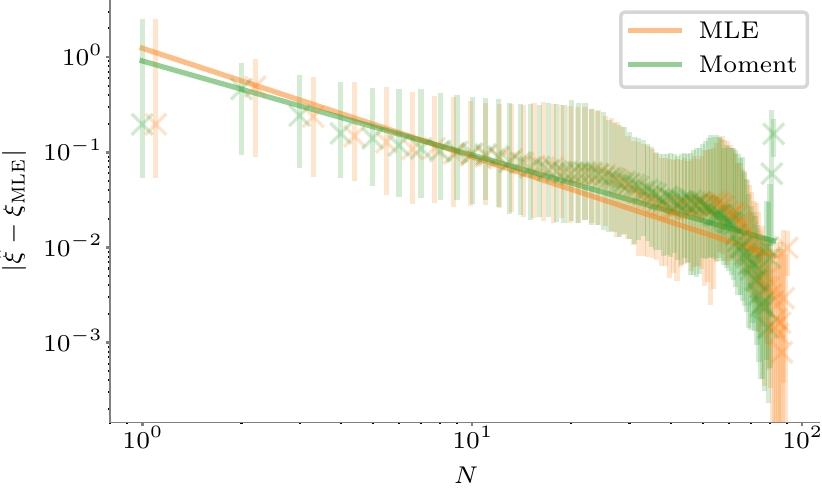}
    \caption{Performance of the Moment Matching estimator and MLE estimator on real-data when compared against the best possible MLE estimate for each webpage (indicated by $\xi_{MLE}$. The sampling policy was decided by the Bing web-crawler and is close to $\rho \approx 0.5$.}
    \label{fig:real-data-1000}
\end{figure}

For the real-data, we take a sample of 1000 webpages from the MSMACRO dataset~\citep{kolobov2019optimal}.
The dataset was crawled by Bing over a 14 week period with each page being visited roughly once every 2 days, \ie, $\rho \approx 0.5$ updates/day.
The best estimate we can make is given by the MLE estimate at the end of the 14 day period. We denote this as $\xi_{\textsc{mle}}$.
Figure~\ref{fig:real-data-1000} shows that the error in estimates obtained after $N$ observations. We see a similar trend as was seen in the synthetic experiments.

\section{Parameter Estimation with Full Observations}\label{sec:full-observability}

In this section, we consider the full observability setting, \ie, when instead of observing whether the page changed or not, one can observe the total number of changed which happened to the webpage in between two observations.
This is often the case while retrieving, \eg, ATOM/RSS feeds~\citep{sia2007efficient}.

The total number of events which a Poisson process with rate $\xi$ produces in time $t$ is a Poisson random variable with mean $\xi \times t$. Additionally, we have:

\begin{lemma}[Adapted from~\citet{pollard2015few}] \label{lemma:poisson-pollard}
If $X$ is a random variable with the distribution $\text{Poisson}(\lambda)$, then:
\begin{align*}
    \prob\left( X > \lambda + \varepsilon \right) &\le \exp\left( -\frac{\varepsilon^2}{2\lambda} \psi\left( \frac{\varepsilon}{\lambda} \right) \right) \\
    \prob\left( X < \lambda - \varepsilon \right) &\le \exp\left( -\frac{\varepsilon^2}{2\lambda} \psi\left( - \frac{\varepsilon}{\lambda} \right) \right) %
\end{align*}
where $\psi(t) = \frac{(1 + t) \log( 1 + t ) - t }{\nicefrac{t^2}{2}}$ for $ t \neq 0$ and $\psi(0) = 1$.
\end{lemma}

Assume that we are given a finite sequence of observation times $\{ y_0 := 0 \}\, \cup\, (y_{n})_{n=1}^{N}$ in advance.
Define $x_{i}$ the number of events of the Poisson process which we observe taking place in the interval $[y_{i-1}, y_i)$.
One can contrast it with the partial observability setting by comparing it with the definition of $o_n$ given in~\eqref{eqn:observation}.
We will use the following estimator:
\begin{align}
\txi = \frac{1}{y_N} \sum_{i=1}^N x_i \label{eqn:xihat-full-definition}
\end{align}
and then clip it to obtain $\hxi = \max{\{ \xi_{\min}, \min{\{\xi_{\max}, \txi \}} \}}$.

\begin{restatable}[]{lem}{concentrationfull}
\label{lemma:concentration-full}
Given observations times $(y_n)_{n=0}^N$ and observation of number of events $(x_n)_{n=1}^N$ for a Poisson process with rate $\xi$, for any $\delta \in (0, 1)$, it holds that
\begin{align*}
    \prob\left( |\hxi - \xi| > \sqrt{\frac{2\xi_{\max}}{y_N \Mcal} \log{\frac{2}{\delta}}} \right) < \delta
\end{align*}
where $\hxi = \max\{\xi_{\min}, \min\{\xi_{\max}, \txi\}\}$ and
$\txi$ is obtained by solving~\eqref{eqn:xihat-full-definition}, and $\Mcal = \psi\left( \frac{\xi_{\max}}{\xi_{\min}} - 1 \right)$.
\end{restatable}
\begin{proof}
We know that $\txi = \sum_{i=1}^{N} x_i \sim \text{Poisson}(\xi \times y_N)$. Then using~\cref{lemma:poisson-pollard}, substituting $\varepsilon$ by $\varepsilon \times y_N$, we get
\begin{align}
    \prob\left( \txi < \xi + \varepsilon \right) &\le \exp\left( -\frac{\varepsilon^2 y_N}{2 \xi} {\psi\left(\frac{\varepsilon}{\xi}\right)} \right)\label{eqn:full-obs-lb}\\
    \prob\left( \txi > \xi - \varepsilon \right) &\le \exp\left( -\frac{\varepsilon^2 y_N}{2 \xi} \psi\left(-\frac{\varepsilon}{\xi}\right) \right)
    \label{eqn:full-obs-ub}
\end{align}

Since $\hxi = \max{ \{ \xi_{\min}, \min{ \{ \xi_{\max}, \txi \} } \} }$, we have $\xi, \hxi \in [\xi_{\min}, \xi_{\max}]$ and, hence,
\begin{align}
    |\txi - \xi| &\ge | \hxi - \xi | \label{eqn:bound-hxi-txi}\\
    \text{ and } \quad | \hxi - \xi | &\le \xi_{\max} - \xi_{min} \label{eqn:max-epsilon} \\
    \implies \frac{| \hxi - \xi |}{\xi} &\le \left( \frac{\xi_{\max}}{\xi_{\min}} - 1 \right) \label{eqn:max-epsilon-ratio}
\end{align}
Define $\varepsilon' = \min{ \{\varepsilon, \xi_{\max} - \xi_{min} \} }$.

It is easy to show that $\psi(t)$ is a decreasing function of $t$ and, therefore, we have:
\begin{align}
    \psi\left( - \frac{\varepsilon'}{\xi} \right) \ge \psi\left( \frac{\varepsilon'}{\xi} \right) \ge \underbrace{\psi \left( \frac{\xi_{\max}}{\xi_{\min}} - 1 \right)}_{\Mcal}
\end{align}
where we have used inequality~\eqref{eqn:max-epsilon-ratio} in the last step.

The equations~\eqref{eqn:full-obs-lb} and~\eqref{eqn:full-obs-ub} can be combined and written as the following:
\begin{align*}
    \prob\left( |\hxi - \xi| > \varepsilon \right) &\le \prob\left( |\hxi - \xi| > \varepsilon' \right) \\
    &\le 2\exp{\left( -\frac{\varepsilon'^2 y_N}{2\xi_{\max}}  \psi\left(\frac{\varepsilon'}{\xi_{\min}}\right) \right)} \\
    &\le 2\exp{\left( -\frac{\varepsilon^2 y_N \Mcal}{2\xi_{\max}} \right)}
\end{align*}
where the first and the last inequalities follow since $\varepsilon' \le \varepsilon$.

Setting $\delta = 2\exp{\left( -\frac{\varepsilon^2 y_N \Mcal}{2\xi_{\max}} \right)}$ and solving for $\varepsilon$, we get the desired result.
\end{proof}

For the uniform-interval exploration till time $\tau$, by using a union bound over all web-pages with~\cref{lemma:concentration-full}, one can obtain:
\begin{align}
    \prob\left( \forall i \in [m]: |\hxi_i - \xi_i| \le \sqrt{\frac{2\xi_{\max}}{\tau \Mcal} \log{\frac{2m}{\delta}}} \right) \ge 1 - \delta
    \label{eqn:full-union-bound}
\end{align}

One can compare~\eqref{eqn:full-union-bound} with~\cref{lemma:uiparameterconcentration} to see the benefits of having full observability, \textit{viz.}, (i) the bound does not depend on the bandwidth $R$ as even a single observation at time $\tau$ provides the sufficient statistic for parameter estimation, and (ii) the dependence of the bound on $\xi_{\max}$ is approximately $\sim \Ocal(\sqrt{\xi_{\max}})$ instead of $\sim \Ocal(e^{c\xi_{\max}})$.
However, since the dependence of the bound on the exploration time remains $\Ocal(\sqrt{\nicefrac{1}{{\tau}}})$, the regret suffered by the ETC algorithm is still $\Ocal(\sqrt{T})$.

\end{appendix}

\end{document}